\documentclass[review,9pt]{article}
\usepackage{amsmath,amsfonts}
\usepackage{algorithmic}
\usepackage{algorithm}
\usepackage{array}
\usepackage{textcomp}
\usepackage{stfloats}
\usepackage{subfigure}

\usepackage{authblk}

\usepackage{url}
\usepackage{verbatim}
\usepackage{graphicx}
\usepackage{amssymb,empheq,bbold}
\usepackage{amsthm}
\usepackage{cases}
\usepackage{booktabs}
\usepackage{multirow}
\usepackage{xcolor}
\usepackage{makecell}
\usepackage{float}

\newtheorem{theorem}{Theorem}

\newtheorem*{remark}{Remark}

\usepackage{hyperref}

\usepackage{geometry}
\newcommand{\keywords}[1]{%
  \par\noindent\textbf{Keywords: }#1\par
}

\title{A Tunable Despeckling Neural Network Stabilized via Diffusion Equation}

\author[1]{Yi Ran}
\author[1]{Zhichang Guo}
\author[1, *]{Jia Li}
\author[1]{Yao Li}
\author[2, 3]{Martin Burger}
\author[1]{Boying Wu}
\affil[1]{School of Mathematics, Harbin Institute of Technology, 150001 Harbin, China.}
\affil[2]{Helmholtz Imaging, Deutsches Elektronen-Synchrotron DESY, 22607 Hamburg, Germany.}
\affil[3]{Fachbereich Mathematik, Universit{\"a}t Hamburg, 20146 Hamburg, Germany.}
\affil[*]{Corresponding author: \texttt{jli@hit.edu.cn}}
\date{}

\begin{document}
\maketitle
\begin{abstract}
    The removal of multiplicative Gamma noise is a critical research area in the application of synthetic aperture radar (SAR) imaging, where neural networks serve as a potent tool. 
	However, real-world data often diverges from theoretical models, exhibiting various disturbances, which makes the neural network less effective. 
	Adversarial attacks can be used as a criterion for judging the adaptability of neural networks to real data, since they can find the most extreme perturbations that make neural networks ineffective.
	In this work, we propose a tunable, regularized neural network framework that unrolls a shallow neural denoising block and a diffusion regularization block into a single network for end-to-end training.	
	The linear heat equation, known for its inherent smoothness and low-pass filtering properties, is adopted as the diffusion regularization block.	
    The smoothness of our outputs is controlled by a single time step hyperparameter that can be adjusted dynamically.
	The stability and convergence of our model are theoretically proven. 
	Experimental results demonstrate that the proposed model effectively eliminates  high-frequency oscillations induced by adversarial attacks.	 
	Finally, the proposed model is benchmarked against several state-of-the-art denoising methods on simulated images, adversarial samples, and real SAR images, achieving superior performance in both quantitative and visual evaluations.	
\end{abstract}




\section{Introduction}
Synthetic aperture radar (SAR), an active remote sensing system, is indispensable for applications such as disaster monitoring, land cover classification, and object detection \cite{liu2017hybrid}.
However, SAR images are often contaminated by speckle noise caused by scattering and coherence phenomena \cite{goodman1976some}.	
Therefore, denoising SAR images is crucial to further analysis or application.	
The acquisition of SAR images can be modeled as \cite{goodman1976some}
\begin{equation}\label{noisemodel}
    f=u \eta,
\end{equation}
where $u:\Omega \subset \mathbb{R}^2  \rightarrow \mathbb{R} $ represents the potential clean image, 
$f$ represents the observed image, and 
$\eta$ represents the multiplicative Gamma noise following the Gamma distribution with mean value equals to one. 
The probability density function related to the multiplicative Gamma noise is given by:
$$
P(\eta)=\frac{L^{L}}{\Gamma(L)} \eta^{L-1} e^{-L \eta} \mathbf{1}_{\{\eta \geqslant 0\}},
$$
where $\mathbf{1}_{\{\eta \geqslant 0\}}$ is the indicator function defined on $\{\eta \geqslant 0 \}$, and $L$ represents the number of looks. 
Current multiplicative denoising methods fall into two categories: traditional methods and deep learning methods.

Early approaches primarily relied on spatial filtering techniques, such as the Lee filter \cite{lee1980digital}, Kuan filter \cite{kuan1985adaptive} and Frost filter \cite{frost1982model}.
Subsequently, variational methods became popular because of their stability and computational efficiency.	
The AA model \cite{aubert2008variational} performed multiplicative denoising by minimizing a functional composed of a total variation (TV) regularizer and a fidelity term obtained by maximum a posteriori (MAP). 
This model had significant influence, though its fidelity term becomes non-convex when $2f < u$. 
To address this issue, the globally convex SO model \cite{shi2008nonlinear} utilized a logarithmic transformation to convert multiplicative noise into additive noise, enabling easier removal. After that, numerous variational models were proposed 	
\cite{Hindawi, baraha2023synthetic, kang2023sar}. 
Recently, anisotropic diffusion equations have become increasingly popular for denoising.	A doubly degradation framework  incorporating a gray value indicator function have been proposed  in \cite{DD}, which controlled the diffusion speed by degenerating in the areas of zero gray values ($u= 0$) and edges ($\left|\nabla u\right|= \infty$) to effectively remove noise.
Numerous partial differential equation (PDE)-based models have been developed under this framework \cite{laghrib2024image, yao2019multiplicative, shan2019multiplicative}, yielding remarkable results.	
Traditional methods offer theoretical completeness and robustness.

With the development of computational power, deep learning-based restoration has significantly outperformed traditional methods.
One of the earliest attempts to utilize neural networks for the removal of multiplicative Gamma noise was SAR-CNN \cite{chierchia2017sar}.	
SAR-CNN employed logarithmic transformation and incorporated the DnCNN module \cite{zhang2017beyond} to perform multiplicative denoising.	
IDCNN was proposed in \cite{wang2017sar}, which divided noisy images by the estimated noise and adopted the structure of residual networks to remove the multiplicative noise. 
However, these conventional neural networks lack interpretability.
To {combine} the power of neural networks {with} the high  interpretability of traditional methods, 
some strategies have been proposed, such as embedding the neural network priors into traditional variational models \cite{xie2022model}, learning parameters in variational model \cite{wang2024variable}, {and using}
the Plug-and-Play \cite{shi2019ppr, baraha2020sar} and {unrolling strategies} \cite{chen2024deep}. 
Models under the Plug-and-Play framework had been successfully applied to SAR denoising \cite{mendes2024robustness, shen2022coupling}.	
Unlike the Plug-and-Play, the unrolling suggested that traditional iterative algorithms can be unfolded into neural networks.	
SAR-RDCP \cite{shen2020sar} employed the half-quadratic splitting method to handle an energy functional, with a neural network regularizer, then unrolled it into the network.

Although deep learning methods have achieved great success, natural disturbances like rain, snow, and fog inevitably degrade the performance of trained neural networks \cite{mei2024comprehensive}.	
This occurs because real SAR data distribution differs from simulated data due to some small perturbations. 
Specifically, the most detrimental perturbations to neural networks are often identified via adversarial attack strategies, for which numerous methods have been developed \cite{goodfellow2014explaining, wei2023adversarial, ning2023evaluating}.	
Moreover, adversarial attacks exhibit a universal characteristic \cite{yan2022towards}: most general neural networks are remarkably vulnerable to such perturbations.
Therefore, adversarial-resistant neural networks show better real-world adaptability, suggesting that adversarial resistance can serve as a metric for evaluating SAR data processing capability.
\begin{figure}[htbp] 
    \centering  
    \includegraphics[width=0.5\textwidth]{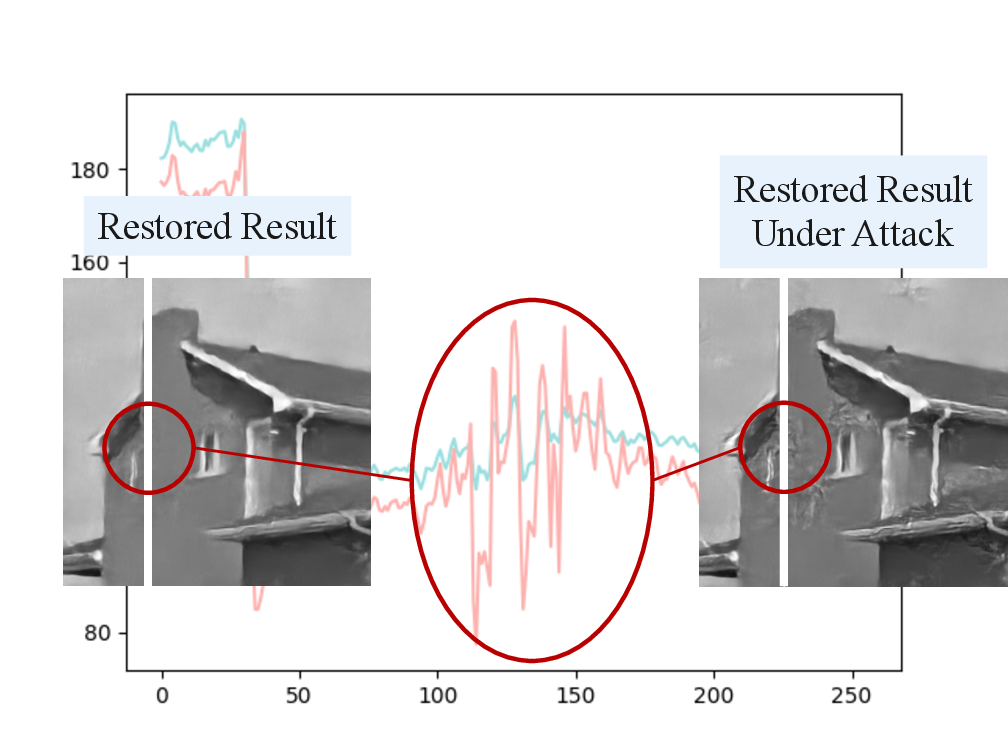}
    \caption{High frequency oscillations generated by adversarial attack. The 70th column of `House' is shown in the center. The restored results on noisy image and adversarial sample are placed in the left and right  respectively.}
    \label{ill_adv_nadv_one_dim}
\end{figure}

Neural networks generate high-frequency oscillations under adversarial attacks.
{The denoising projected gradient descent (denoising-PGD) attack algorithm \cite{ning2023evaluating}, which is specifically designed for image denoising and demonstrates remarkable transferability across neural networks, is employed to generate adversarial samples. These samples simulate extreme cases of real SAR data.}
To more clearly illustrate the impact of adversarial attacks on neural networks, AGSDNet \cite{thakur2022agsdnet} is used to process noisy image `House' and corresponding adversarial sample. The results are shown in Figure \ref{ill_adv_nadv_one_dim}. 
The core of dealing with adversarial samples is to eliminate the high frequency oscillations in the red ellipse in Figure  \ref{ill_adv_nadv_one_dim}.

Robustness is a fundamental requirement in denoising algorithms.
Figure \ref{ill_adv_nadv_one_dim} illustrates that the neural network is unstable under adversarial perturbations. To improve network robustness, a common method is adversarial training, such as Lipschitz learning \cite{bungert2021clip}. 	 
However, computing the Lipschitz constant of a neural network is a  non-deterministic polynomial (NP)-hard problem \cite{virmaux2018lipschitz}.	
Traditional methods {such as} variational and PDE-based approaches have well-established theoretical foundations that guarantee stability, even though their restoration performance is significantly worse than {that of} neural networks.	
Numerous methods combining traditional approaches with neural networks have been proposed.	
Extensive theoretical analysis guarantees the convergence of plug-and-play methods, but the Lipschitz constant constraint in their modeling often fails in many networks \cite{terris2024equivariant}.	
Moreover, plug-and-play methods are ineffective at recovering adversarial samples \cite{ning2023evaluating}.	
The two-stage IP-NDE method integrates neural network priors into diffusion equation coefficients to improve robustness \cite{cheng2024diffusion}.	
The main stabilizing mechanism of this two-stage approach lies in the diffusion equation, whereas its effect on neural network robustness remains indirect.

Recently, the concept of neural ordinary differential equations (Neural ODEs) was proposed	
\begin{equation}\label{eq:resnet}
    u^{k+1}=u^{k}+f\left(u^{k}, \theta_k\right),
\end{equation}
where $\theta_k$ is the trainable parameters, and $f$ is a function constructed by some network structures. In fact, \eqref{eq:resnet} represents the discrete scheme of an ODE, and can be seen as a residual block in neural network. 
Subsequently, neural stochastic differential equations (Neural SDEs), which model stochastic noise injection through stochastic differential equations, were proposed in \cite{liu2019neural} to enhance neural network robustness.	
Additionally, the authors of \cite{ruthotto2020deep} studied the following problem:	
\begin{equation}\label{eq:pde_net}
        \begin{cases}
    \partial_t \mathbf{Y}(\boldsymbol{\theta}, t)  =\mathbf{F}(\boldsymbol{\theta}(t), \mathbf{Y}(t)), \text { for } t \in(0, T], \\
    \mathbf{Y}(\boldsymbol{\theta}, 0)  =\mathbf{Y}_0 ,
        \end{cases}
\end{equation}
where $\mathbf{Y}=\left[\mathbf{y}_1, \mathbf{y}_2, \ldots, \mathbf{y}_s\right] \in \mathbb{R}^{n \times s}$, $\mathbf{F}$ is some neural networks, and $\boldsymbol{\theta}$ represents learnable parameters. 
By considering different scenarios, the authors of \cite{ruthotto2020deep} developed parabolic, hyperbolic, and second-order hyperbolic CNNs. 
Essentially, the neural PDE networks learn feature transformations over the interval $[0, T]$.

The PDE networks proposed in \cite{ruthotto2020deep} satisfy the stability condition	
\begin{equation}\label{eq:sta}
\|\mathbf{Y}(\boldsymbol{\theta}, T)-\tilde{\mathbf{Y}}(\boldsymbol{\theta}, T)\|_F \leq M\|\mathbf{Y}(0)-\tilde{\mathbf{Y}}(0)\|_F,
\end{equation}
under certain assumptions. In \eqref{eq:sta}, the constant $M > 0$ is  independent of $T$, where $\mathbf{Y}$ and $\tilde{\mathbf{Y}}$ are solutions of \eqref{eq:pde_net} for different initial values and $\|\cdot\|_F$ is the Frobenius norm.  However, when 
$M>1$, any disturbance in the initial value may be amplified in the solution.

In this paper, we propose a tunable robust despeckling network framework that integrates diffusion equations into the network multiple times and trains them jointly to enhance the robustness of the entire neural network.	
The dissipation phenomenon characterized by diffusion equations occurs spontaneously, {explains} their insensitivity to noise distribution and the underlying mechanism, as shown in Figure \ref{dissp}.	
This universality highlights the inherent stability of diffusion equations, motivating us to use them as regularizers to improve neural network robustness.	
Unlike neural differential equation networks, our model evolves {the} learned features through deterministic equations even beyond time $T$. 
Moreover, when using the heat equation for evolution and assuming the denoising network block is non-expandable, the output of our model is non-expanding, meaning	
\begin{equation*}
   \| \mathcal{W}_{\Theta}(f) - \mathcal{W}_{\Theta}(\tilde{f})\| \leqslant \| f-\tilde{f}\|,
\end{equation*}
where $f$ and $\tilde{f}$ are inputs with minimal differences, and $\mathcal{W}_{\Theta}$ represents the whole network. 
The main contributions of our model are as follows:
\begin{enumerate}
    \item A strategy of embedding the diffusion equation into the neural network multiple times for joint training is proposed to improve its robustness. Inside the network, a deterministic equation can be used to evolve the learned features.
	\item Compared with adversarial training, there is no need to add additional adversarial samples to the training set, which reduce the amount of calculation.
	\item{Adjusting the only one parameter in our trained model can control the smoothness of the outputs to accommodate different degrees of disturbances from the ideal model  \eqref{noisemodel}.}
\end{enumerate}

The rest of this paper is organized as follows. Section \ref{sec:2} introduces the instability of neural networks  and some properties of {the} heat equation. In section \ref{sec:3}, the framework and {details} of our model are presented firstly. Then, the convergence analysis and the numerical algorithm are also given in section \ref{sec:3}. The experiments on simulated images, different  adversarial samples and real SAR images are shown in section \ref{sec:exp} to validate the performance of {our} model. The ablation study, model decoupling, and the comparison with adversarial training are also discussed in section \ref{sec:exp}. Finally, {the} conclusion is provided in section \ref{sec:5}.

\section{Related Work}\label{sec:2}

\subsection{Adversarial attack and network instability}
Adversarial attacks exploit minor perturbations in input data to significantly degrade the performance of network \cite{goodfellow2014explaining}. 
In this paper, we use  adversarial attacks to test the robustness of the model to improve the adaptability of our model to real data. 

Recently, a specialized attack on the denoising algorithm, $L^2$-denoising-PGD method \cite{ning2023evaluating} minimized the negative $L^2$-norm between the output of the network  and the clean image to complete the attack. 	
This method exhibits strong transferability, that a single adversarial sample can compromise the denoising performance of multiple neural networks.
Suppose $f$ and $u$ are the noisy image and the clean image, respectively. For a denoising network $\mathcal{W}_{\theta}(\cdot)$, 
the adversarial attack is formulated as	

\begin{equation*}
	\underset{\varepsilon}{\operatorname{max}} \ 
	\mathcal{L}(\mathcal{W}_{\theta}(f+\varepsilon), u),  
	\quad \text { s.t. }\|\varepsilon\|_{2} \leqslant \epsilon,
\end{equation*}
where $\mathcal{L}$ is the loss function and $\epsilon$ is the threshold of the perturbation. 
The process of generating adversarial samples under the $L^2$-denoising-PGD strategy is	
\begin{equation*}
	f^{t+1}=f^t+\alpha \frac{\nabla_f \mathcal{L}(\mathcal{W}_{\theta}(f), u)}{\left\|\nabla_f \mathcal{L}(\mathcal{W}_{\theta}(f), u)\right\|_2},
\end{equation*}
where $f^t$ is the adversarial sample at iteration $t$, $\alpha$ represents the iterating step. In addition, the $L^\infty$-denoising-PGD  is	also a common method, which can be formulated as 
\begin{equation*}
	\underset{\varepsilon}{\operatorname{max}} \ 
	\mathcal{L}(\mathcal{W}_{\theta}(f+\varepsilon), u),  
	\quad \text { s.t. }\|\varepsilon\|_{\infty} \leqslant \epsilon.
\end{equation*}
Then, its iterative update strategy is
\begin{equation*}
	f^{t+1}=f^t+\alpha \operatorname{sign}\left(\nabla_f \mathcal{L}(\mathcal{W}_{\theta}(f), u)\right),
\end{equation*}
where $f^t$ is the adversarial sample at iteration $t$, $\alpha$ represents the iterating step and $\operatorname{sign}(\cdot)$ is the sign function. 
To maintain a sufficiently small perturbation, values are clipped to remain within the range $[-\epsilon, \epsilon]$.

\begin{figure*}[htbp]
	\centering
	\subfigure[$t=20$]{\includegraphics[width=0.24\columnwidth]{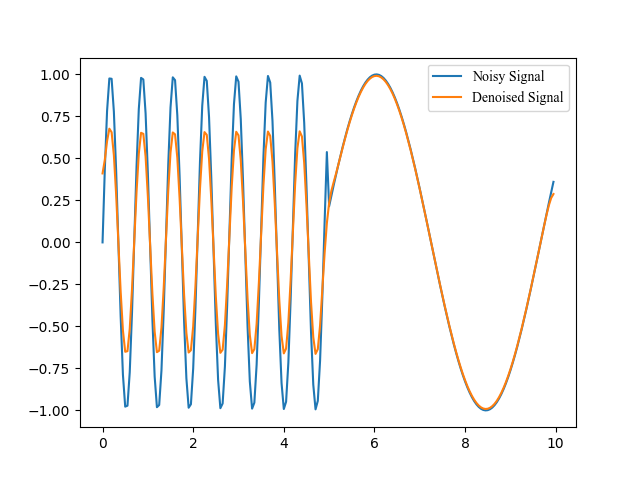}} 
	\subfigure[$t=40$]{\includegraphics[width=0.24\columnwidth]{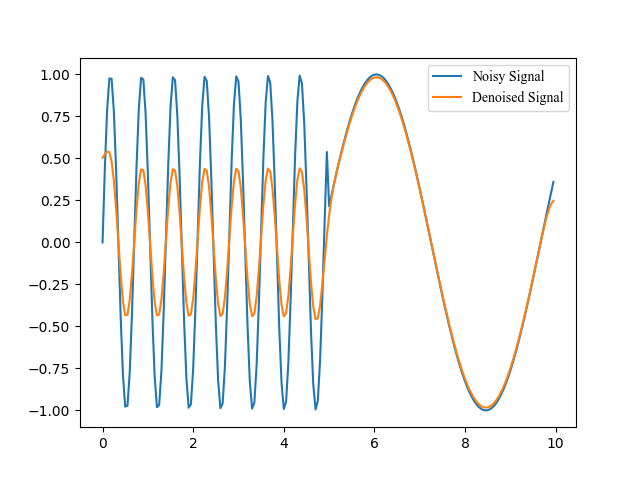}} 
	\subfigure[$t=60$]{\includegraphics[width=0.24\columnwidth]{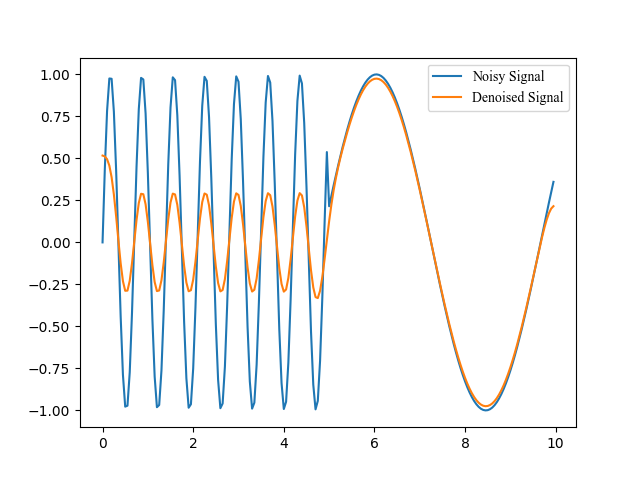}} 
	\subfigure[$t=80$]{\includegraphics[width=0.24\columnwidth]{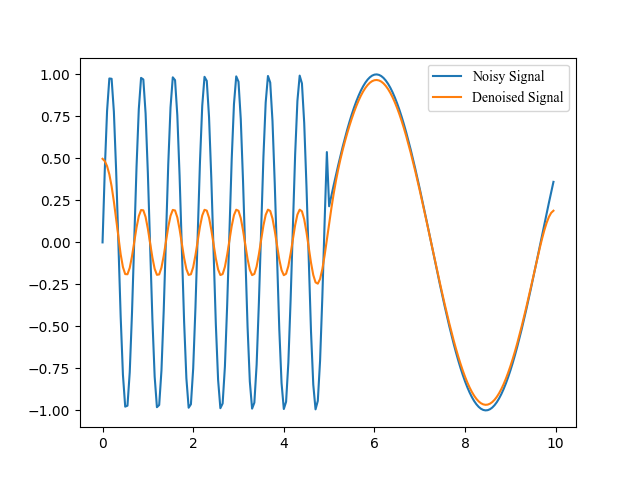}} 
	\caption{The results of the heat equation for different frequency signals as the time increases.}
	\label{itr_heat}
\end{figure*}

Adversarial behavior is inevitable in neural network models. 
SAR-CNN \cite{chierchia2017sar} is one of the earliest pure data-driven despeckling neural networks and has competitive effect. 
SAR-RDCP \cite{shen2020sar} is a typical despeckling model that combines data-driven and model-driven strategies, which has high interpretability and performance. 
Although the performance of these two models both have declined when subjected to adversarial attacks, which can be observed in Figure \ref{adversarial_L=1}, 
the oscillations produced by SAR-RDCP are far less severe than those produced by SAR-CNN. 
This is because SAR-RDCP is not entirely a network, but incorporates a hand-crafted fidelity term that provides some protection against attacks.

The denoising algorithm should be a stable process of eliminating high-frequency oscillations. 
From Figure \ref{ill_adv_nadv_one_dim}, neural networks are very sensitive to disturbances, that produce artificial effects and even amplify the disturbances. 
Besides, the adversarial attacks have little impact on traditional methods \cite{ning2023evaluating}, such as BM3D, variational methods and PDE methods. 
Therefore, the denoising network is actually unstable and this is why it is vulnerable to attack.

\subsection{Regularity of Heat equation}
The regularity of diffusion equations includes the smoothness and differentiability of solutions over time and space. 
In this study, we leverage the regularity of diffusion equations to mitigate adversarial attack impacts and enhance neural network robustness.
Here, we outline the solution properties of the heat equation \cite{evans2022partial}, which serve as a foundation for later analysis. 	
\begin{theorem}\label{fundamen_solu}
	Assume $f \in C(\mathbb{R}^n) \cap L^\infty(\mathbb{R}^n) $, then the solution of the following initial-value 
	\begin{equation*}
		\begin{cases}
			\frac{\partial u}{\partial t} = \Delta u, \text{ in } \mathbb{R}^n \times (0,\infty), \\
			u(x,0)=f, \text{ on } \mathbb{R}^n \times \{t=0\},
		\end{cases}
	\end{equation*}
	is
	\begin{equation}\label{heat_solution}
	u(x,t)=(\Phi*f)(x,t)=\frac{1}{(4\pi t)^{n/2}}\int_{\mathbb{R}^n }e^{-\frac{|x-y|^2}{4t}}f(y)\mathrm{d}y,
	\end{equation}
	where 
	$\Phi(x, t)= \frac{1}{(4 \pi t)^{n / 2}} e^{-\frac{|x|^2}{4 t}} $ for $ x \in \mathbb{R}^n$ and $ t>0$ is called the fundamental solution and $u \in C^\infty(\mathbb{R}^n \times (0,\infty))$. 
\end{theorem}

The proof of Theorem \ref{fundamen_solu} can be found in \cite{evans2022partial}. When the noisy image $f$ satisfies the condition in Theorem \ref{fundamen_solu}, the solution of heat equation is infinity smoothness, which provides sufficient regularity for our model. It also shows that heat equation denoising is a stable algorithm. 

The low-pass property of \eqref{heat_solution} that attenuates high frequencies in a monotone way \cite{weickert1998anisotropic} is another important property that can be shown in Figure \ref{itr_heat}. From Figure \ref{itr_heat}, the heat equation prioritizes the high-frequency part of the signal, which will gradually be smoothed out as time increases.

\section{Methodology}\label{sec:3}

\subsection{Framework of our model}
Stability is one of the most important conditions that a denoising algorithm needs to satisfy. 
The powerful data-fitting ability of neural networks makes them very easy to overfit, which greatly affects the robustness of neural networks. 
The instability of neural networks makes them ineffective when the real SAR data does not completely match the simulated data. 

The adversarial attacks can {evaluate} the robustness of a neural network. 
The real data has more disturbances than the simulated data, such as rain and snow.
{Specifically}, finding the most ineffective direction of neural networks and adding small perturbations is the core of adversarial attacks. 
In addition, the networks with stronger denoising effects on simulated data tend to be less {robust}, which can be observed from Figure \ref{adversarial_L=1} in subsection \ref{subsec:result_adv}. 
Therefore, adversarial attack is an important method to detect the adaptability of neural networks to real data.

Figure \ref{ill_adv_nadv_one_dim} shows that small perturbations can introduce high-frequency oscillations, significantly degrading the denoising performance.
However, the regularity of the heat equation solution, as established in Theorem \ref{fundamen_solu}, and its demonstrated low-pass filtering property in Figure \ref{itr_heat}, motivate{s the} adoption of diffusion equations to enhance the robustness of neural networks.

In this paper, we propose a novel tunable neural network framework that integrates the high performance of neural networks with the inherent stability of diffusion equations, providing robustness against adversarial attacks by modulating the smoothness of outputs.	
Formulaically, we incorporate
a diffusion equation $\frac{\partial u}{\partial t} = \mathcal{H}(u,\nabla u,\nabla^2 u)$  as a regularizer in the denoising network and unroll it to propose our framework. 	
For $k=0,1,\cdots, K$,  formulating our model as
\begin{empheq}[left=\empheqlbrace]{align}
	&z^{k+1}=\mathcal{D}_\Theta\left(u^k\right), \label{AANN}\\
	&u^{k+1}=z^{k+1} + \tau   \mathcal{H}(u^{k+1},\nabla u^{k+1}, \nabla^2 u^{k+1}), \label{AAlaplace}
	\end{empheq}
where $u^{0}=f$ is the noisy image, $\mathcal{D}_\Theta(\cdot)$ is a denoising neural network block, and \eqref{AAlaplace} is the implicit scheme of diffusion equation with initial value  $z^{k+1}$ and time step $\tau$. 

The unrolling algorithm is used to unfold the iterative process within a neural network, allowing end-to-end training with pairs of noisy images and their corresponding labels.	
The framework of our proposed model is illustrated in Figure \ref{frmae1}. 

\begin{figure*}[htbp] 
    \centering  
    \includegraphics[width=1\textwidth]{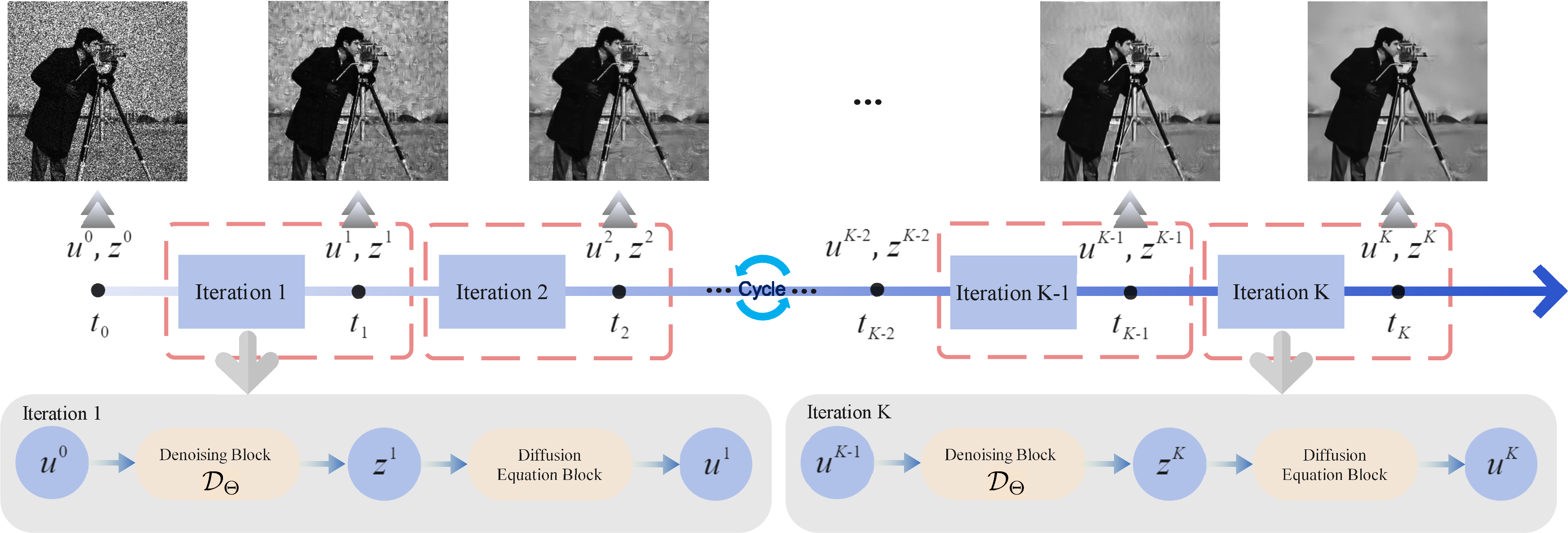}
    \caption{framework of our model.}
    \label{frmae1}
\end{figure*}

Notably, our model can be regarded as a fitting-correction system.	
The denoising neural network block $\mathcal{D}_\Theta(\cdot)$ exhibits strong fitting ability, while the diffusion regularization block \eqref{AAlaplace} enhances the smoothness of the denoised results, thereby preventing overfitting and improving stability.	
The core idea of our model is to smooth out the internal shocks caused by adversarial attacks in the network through the dissipation mechanism of the diffusion equation to improve the robustness.

The cumulative regularization effect introduced by the unrolled recursive architecture augments the flexibility of our model, effectively modulating the smoothness of outputs.	 
Our model allows adjustment of the time step $\tau$ at any stage, even after training.
When facing adversarial attacks, a large time step is necessary to remove oscillations, but when the result is too smooth, it is better to reduce the step size to preserve details, as illustrated in Figure \ref{tau_impacet} and Figure \ref{tau_impacet_no_adv}.	
In some sense, one hyperparameter can control the smoothness of the output of the entire model.

It is worth mentioning that  embedding the diffusion equation multiple times into the network is necessary. Simply cascading a denoising network with a diffusion equation fails to eliminate adversarial-induced oscillations.  
To demonstrate this, we cascade the 17-layer SAR-CNN with a heat equation and train it on the same dataset {used for our model}.
The time step  during training is also kept the same as our model ($\tau=0.1$). Using an adversarial test image, Figure \ref{cascades} shows restoration results at different $\tau$. 
\begin{figure*}[t]
        \centering
    \begin{tabular}{@{}c@{~~}c@{~~}c@{~~}c@{}}
		\includegraphics[width=.2\textwidth]{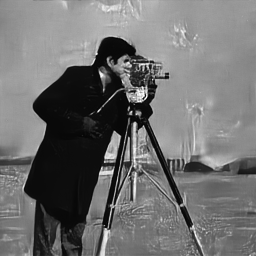} &
		\includegraphics[width=.2\textwidth]{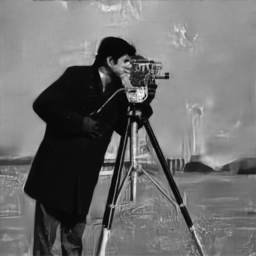} &
        \includegraphics[width=.2\textwidth]{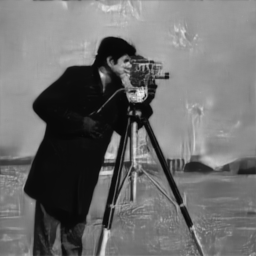} &
		\includegraphics[width=.2\textwidth]{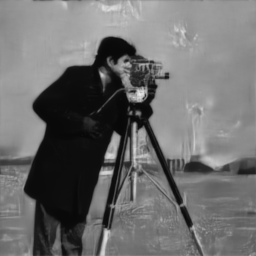} \\
        
		\footnotesize{\tiny $\tau=0.1$} & 
		\footnotesize{\tiny $\tau=0.3$} & 
		\footnotesize{\tiny $\tau=0.5$}  & 
		\footnotesize{\tiny $\tau=0.7$} \\
        \footnotesize{\tiny PSNR=24.09, SSIM=0.77 } & 
		\footnotesize{\tiny PSNR=23.82, SSIM=0.77} & 
		\footnotesize{\tiny PSNR=23.30, SSIM=0.77}  & 
		\footnotesize{\tiny PSNR=22.86, SSIM=0.76} \\

        \includegraphics[width=.2\textwidth]{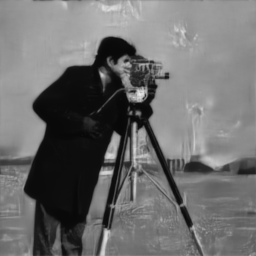} &
		\includegraphics[width=.2\textwidth]{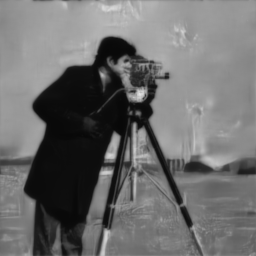} &
        \includegraphics[width=.2\textwidth]{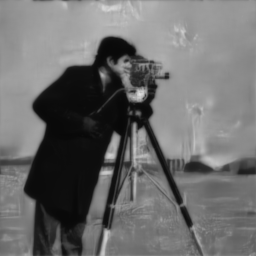} &
		\includegraphics[width=.2\textwidth]{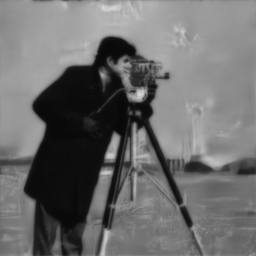} \\
		\footnotesize{\tiny $\tau=0.9$} & 
		\footnotesize{\tiny $\tau=1.1$} & 
		\footnotesize{\tiny $\tau=1.3$}  & 
		\footnotesize{\tiny $\tau=1.5$} \\
        \footnotesize{\tiny PSNR=22.86, SSIM=0.76 } & 
		\footnotesize{\tiny PSNR=22.19, SSIM=0.75} & 
		\footnotesize{\tiny PSNR=21.93, SSIM=0.74}  & 
		\footnotesize{\tiny PSNR=21.71, SSIM=0.74} \\

      \end{tabular}
      \caption{The results of cascade SAR-CNN and heat equation on the adversarial { sample} with variance of $\tau$.}
	\label{cascades}
\end{figure*}
Figure \ref{cascades} shows that both PSNR and SSIM values decrease as $\tau$ increases.
Visually, larger time steps only blur images without removing adversarial-induced oscillations. 
This occurs because the diffusion equation is not properly integrated into the network, preventing it from handling internal adversarial oscillations.

The details about each components of our model are provided below.

\subsection{Diffusion regularization block}\label{subsec:difff}
The dissipation phenomenon of diffusion equations inspires us to design a regularization block. 
The general diffusion equations can be written as 
\begin{equation*}
	\frac{\partial u}{\partial t} =  \operatorname{div} D(u) \nabla u,
\end{equation*}
where $D(u)$ is the diffusion coefficient {that} controls the diffusion velocity and direction. 
The gradients in the distribution of $u$ decrease, indicating a reduction in the disparities between high and low values of $u$.	
This reduction in gradients correlates with energy dissipation within the system.	
Importantly, the dissipative {process} occurs spontaneously and is independent of the noise distribution and its underlying mechanism.	
In other words, diffusion equations reduce the disparity between the maximum and minimum values of the input, regardless of their cause.

It is widely accepted that nonlinear equations offer superior denoising and better edge preservation.	
However, the inherent nonlinearities of neural networks suffice to fully accomplish the denoising task.
Under this condition, the simplest case can better illustrate the effectiveness of our model in enhancing robustness. 
Furthermore, employing linear diffusion equations reduces computational complexity and simplifies theoretical analysis.	
Therefore, the most simplest linear diffusion equation 
\begin{equation*}
    \frac{\partial u}{\partial t} =  \Delta u,
\end{equation*}
is selected as the regularizer, which is called  the heat equation. 
At that time, our model can be concretely described as

\begin{empheq}[left=\empheqlbrace]{align}
&z^{k+1}=\mathcal{D}_\Theta\left(u^k\right), \label{OurNN}\\
&u^{k+1}=z^{k+1} + \tau  \Delta u^{k+1}. \label{Ourlaplace}
\end{empheq}
Other reasons for choosing the heat equation as the regularizer are explained below.

The sufficient regularity and the low-pass property of heat equation are two additional reasons for adopting it as regularizer. 
The low-pass property ensures that high-frequency oscillations generated by adversarial attacks within the network can be quickly eliminated and the infinite smoothness of the solution to the heat equation provides sufficient energy to improve the overall robustness of our neural  network.

\begin{figure}[htbp] 
    \centering  
    \includegraphics[width=0.45\textwidth]{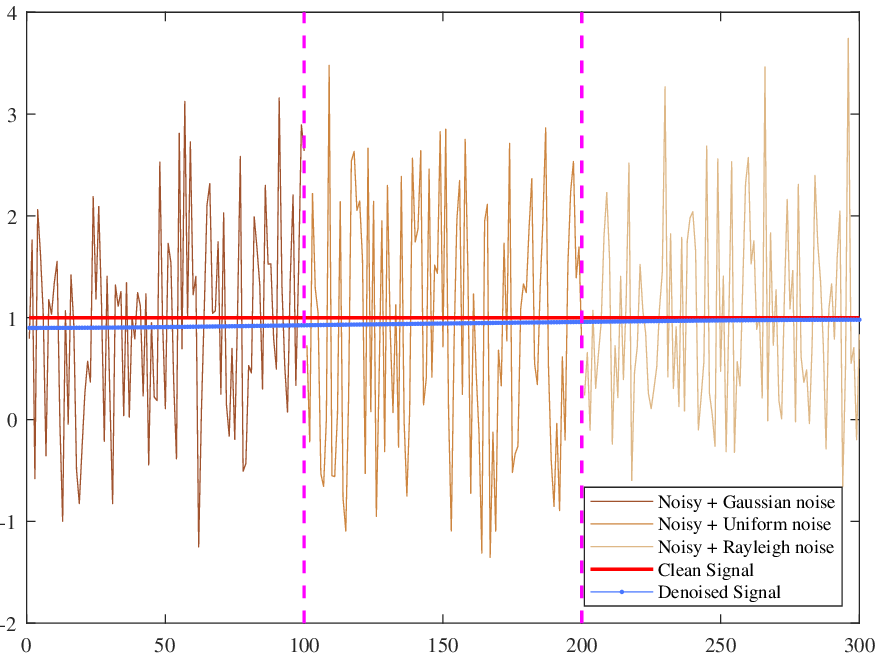}
    \caption{The performance of heat equation to different noise. The "Noisy" signals are constructed by multiplying Gamma noise with $L=10$ to the "Clean Signal". Adding "Gaussian noise",  "Uniform noise", and "Rayleigh noise" to 0-100, 100-200 and 200-300 respectively to simulate three different attacks.  The processed "Noisy" signal is "Denoised Signal".}
    \label{dissp}
\end{figure}

The insensitive to noise distribution of the  dissipative phenomenon can be used to handle a variety of disturbances, not just adversarial attacks. 
To {illustrate} more clearly, we introduce three distinct types of disturbances into the simulated SAR noisy signal and apply the heat equation for signal processing, as illustrated in Figure \ref{dissp}.	
Figure \ref{dissp} demonstrates that the heat equation effectively restores the {signal} (blue line), irrespective of the {different} underlying causes of the disturbances.

The diffusion regularization block is equivalent to a convolutional layer with a fixed convolution kernel, whose parameters remain unchanged  during training. 
Because the heat equation solution $u(x,t)$ is the convolution of the fundamental solution $\Phi$ and the initial value $f$ as shown in \eqref{heat_solution}. 

The reason why we use implicit scheme in \eqref{Ourlaplace} to discrete heat equation is the unconditional stability of implicit scheme, which means that the time step can be changed arbitrary. 
The explicit scheme is the simplest to implement for discreting the PDEs, but its stability is constrained by the Courant-Friedrichs-Lewy (CFL) condition.	
A large time step is necessary to resist severe adversarial attacks, which is the limitation of explicit scheme. 

Consequently, the diffusion function block in \eqref{Ourlaplace} accumulatively mitigates distortions inside the networks, which  induced by adversarial attacks.

\subsection{Denoising neural network block}

Basically,  any denoising CNN can be selected as the denoising block in our model, but, the only one standard is that it is not too deep. The reasons are explained below. 

A shallow architecture for each block is sufficient for effective denoising under the recursive unrolling architecture.	
Deeper neural networks yield better denoising effect, but the existence of the recursive structure introduces more calculations.

Our main goal is {to increase} the robustness of the neural network for better adaptation to  real-world data.
Deeper layers induce more severe oscillations inside the network, necessitating larger time steps to ensure adequate regularity to counteract the oscillations induced by adversarial attacks. Larger time steps typically lead to increased errors in  the discrete scheme of the diffusion equation. 

Specifically, transforming multiplicative noise into additive noise via a logarithmic transformation is a natural idea  
\begin{equation*}
    \log f=\log u+\log \eta.
\end{equation*}
Then, a shallow DnCNN denoising module \cite{zhang2017beyond} is employed as our denoising block $\mathcal{D}_\Theta(\cdot)$ and its structure is shown in Figure \ref{frmae}.

For details, given a set of training samples ${f_i, y_i} (i=1,\cdots, N)$, where $f_i$ and $y_i$ represent noisy and clean images, respectively, 
the five layers DnCNN denoising module is used as the denoising block $\mathcal{D}_\Theta(\cdot)$. It  consists of five fully convolutional layers  without any pooling. 
The ﬁlter sizes are $3\times3\times64$, and  the input channel and the output channel are both set to 1. 
Since the DnCNN denoising module learns the noise distribution \cite{zhang2017beyond}, a residual connection is established between the first and last layers.	

\begin{figure*}[htbp] 
    \centering  
    \includegraphics[width=1\textwidth]{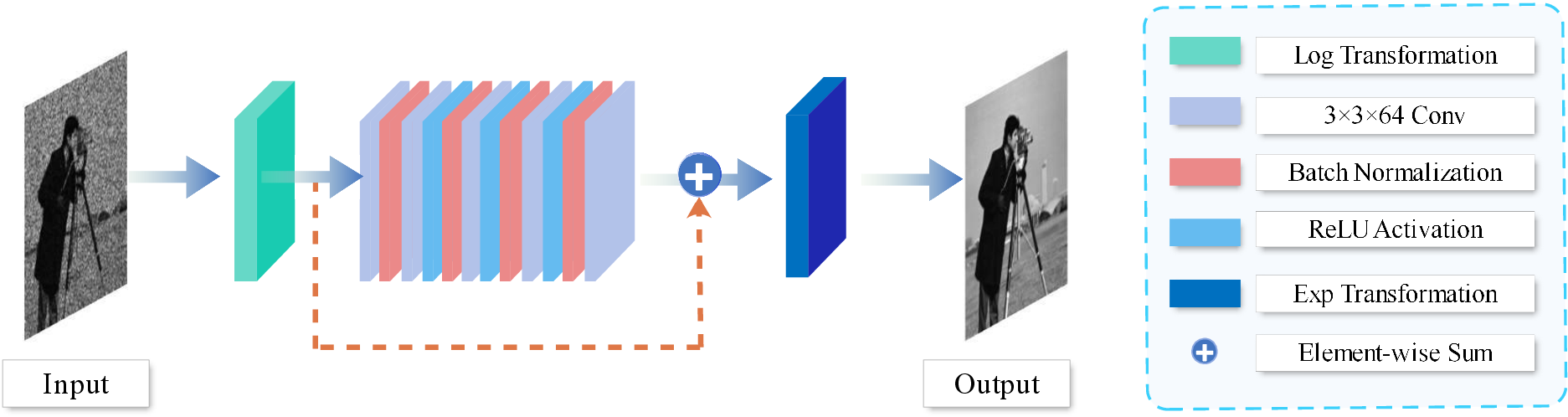}
    \caption{The structure of denoising block.}
    \label{frmae}
\end{figure*}

The entire network structure is obtained by arranging $\mathcal{D}_\Theta(\cdot)$ and heat equations five times in sequence, which is denoted as $\mathcal{W}_\Theta$. Our network can be trained end-to-end, and the loss function \cite{chierchia2017sar} is	
\begin{equation*}
    \mathcal{L} (\Theta)=\sum_{i=1}^N \mathbf{1}^{\top} \log \left(\cosh \left(\log {{y}_i}-\mathcal{W}_{\Theta}\left(\log {f}_i\right)+c\right)\right),
\end{equation*}
where $ c $ is the nonzero mean of log-speckle.

\subsection{Theoretical analysis}

In this section, the convergence and stability of our model can be given in Theorem \ref{conver} and Remark, respectively. 

\begin{theorem}[Convergence]\label{conver}
	Suppose $\Omega$ is a bounded open  domain with Lipschitz boundary $\partial \Omega$, which contains the image domain. 
	Besides, the heat equation satisfies the Dirichlet condition that $u=0$ on $\partial \Omega$. 
	Assume  the neural network $\mathcal{D}_\Theta\left(\cdot\right)$ is non-expandable: $$
	\|\mathcal{D}_\Theta\left(u_1 \right) - \mathcal{D}_\Theta\left(u_2 \right) \|\leqslant  \|u_1-u_2\|. 
	$$ 
	Then, the sequence $\{u^k\}_{k=0}^{K}$ generated by our algorithm is globally convergent.
	
\end{theorem}

\begin{proof}
	The proposed recursive algorithm  can be rewritten as 
	\begin{equation}\label{my_alg}
	u^{k+1} = \mathcal{D}_\Theta (u^k) + \tau  \Delta u^{k+1},
	\end{equation}
	from $k$ step to $k+1$ step, which means that for $u^k \neq v^k$, it follows that 
	\begin{equation}
		(\mathcal{I}-\tau \Delta  )(u^{k+1} - v^{k+1}) = \mathcal{D}_\Theta(u^k) - \mathcal{D}_\Theta(u^k),
	\end{equation}
	where $\mathcal{I}$ is the identity operator.  
	In addition, $\| (\mathcal{I}-\tau \Delta  )^{-1} \| \leqslant \varepsilon$, for some  $0<\varepsilon< 1$ since  $-\Delta$ is positive definite. Combining the non-expandable of $\mathcal{D}_\Theta(\cdot)$, we have 
	\begin{equation}\label{eq:non_exp_ite}
		\begin{aligned}
		\|u^{k+1} - v^{k+1}\| &\leqslant \| (\mathcal{I}-\tau \Delta  )^{-1} \| \|\mathcal{D}_\Theta(u^k) - \mathcal{D}_\Theta(u^k)\| \\
		&\leqslant \| (\mathcal{I}-\tau \Delta  )^{-1} \|  \|u^{k} - v^{k}\| \\
		& \leqslant \varepsilon \|u^{k} - v^{k}\|. 
		\end{aligned}
	\end{equation}
\end{proof}

The stability analysis or the continuous dependence of solutions on initial values can be directly obtained from the  estimate \eqref{eq:non_exp_ite}.  
\begin{remark}[Stability]\label{stab}
	If the neural network $\mathcal{D}_\Theta\left(\cdot\right)$ satisfies the conditions in Theorem \ref{conver}, our algorithm is stable.  Moreover, our model satisfies 
    \begin{equation*}
       \| \mathcal{W}_{\Theta}(f) - \mathcal{W}_{\Theta}(\tilde{f})\|\leqslant \| f-\tilde{f}\|.
     \end{equation*}
\end{remark}

\subsection{Numerical scheme of diffusion block}

\begin{algorithm}[t]
    \scriptsize
	\caption{algorithm of our model}\label{alg:alg1}
	\begin{algorithmic}
	\STATE 
	\STATE {\textsc{TRAIN}}
	\STATE {Input: }$\text{noisy image } f \text{, clean image } y$
	\STATE {Initialization: }
	\STATE \hspace{0.5cm}{(1) Set parameters: $K$, $\tau$ and the parameters in $\mathcal{D}_\Theta\left(\cdot\right)$}
	\STATE \hspace{0.5cm}{(2) Initialize $u^{0}=f$}
	\STATE \hspace{0.5cm}{(3) Acquire input size $N_1$ and $N_2$}
	\STATE {for $k$ from 1 to $K$:}
	\STATE \hspace{0.5cm}$z^{k+1} \gets \mathcal{D}_\Theta\left(u^k\right)$
	\STATE \hspace{0.5cm}{for $x_i$ from 1 to $N_1$ and $y_i$ from 1 to $N_2$: }
	\STATE \hspace{1cm}{$u^{k+1}(x_i,y_j) \gets \mathcal{F}^{-1}\left( \frac{\mathcal{F} z^{k}(x_i,y_j)}{-2\tau(\operatorname{cos}(\frac{2\pi}{N_1} x_i)+\operatorname{cos}(\frac{2\pi}{N_1} y_i)-2)+1} \right)$}
	\STATE {return $u^{k+1}$}  
	\STATE {\textit{loss} $\gets \mathcal{L}(u^{k+1},y)$}  
	\STATE {\textit{loss.back()} to update the parameters in our model }
	\STATE 
	\STATE {\textsc{PREDICT}}
	\STATE {Input: }$\text{noisy image } f $
	\STATE {Initialization: }
	\STATE \hspace{0.5cm}{(1) Set a proper $\tau$ }
	\STATE \hspace{0.5cm}{(2) Initialize $u^{0}=f$}
	\STATE \hspace{0.5cm}{(3) Acquire input size $N_1$ and $N_2$}
	\STATE {for $k$ from 1 to $K$:}
	\STATE \hspace{0.5cm}$z^{k+1} \gets \mathcal{D}_\Theta\left(u^k\right)$
	\STATE \hspace{0.5cm}{for $x_i$ from 1 to $N_1$ and $y_i$ from 1 to $N_2$: }
	\STATE \hspace{1cm}{$u^{k+1}(x_i,y_j) \gets \mathcal{F}^{-1}\left( \frac{\mathcal{F} z^{k}(x_i,y_j)}{-2\tau(\operatorname{cos}(\frac{2\pi}{N_1} x_i)+\operatorname{cos}(\frac{2\pi}{N_1} y_i)-2)+1} \right)$}
	\STATE {return $u^{k+1}$}  
	\end{algorithmic}
\end{algorithm}

Compared with the explicit scheme, the implicit discrete scheme is more difficult to solve. Here, the Fourier transform is used to solve \eqref{Ourlaplace}	quickly. 

Assume that the periodic boundary condition is imposed on the heat equation.	
The shifting operators are denoted by $\mathcal{S}_x^\pm u(i,j)=u(i\pm 1,j)$ and $\mathcal{S}_y^\pm u(i,j)=u(i,j\pm 1)$.  
Consider the implicit scheme for the heat equation in two dimensions
\begin{equation}\label{fft}
	u^{k+1}=u^{k} + \tau( D_x^- D_x^+ u^{k+1} +D_y^- D_y^+ u^{k+1}),
\end{equation}
where the $D_x^\pm $ and $D_y^\pm $ are the forward (backward) operators on the direction $x$  and $y$, respectively. 
Thus, they can be presented as $D_x^- D_x^+ = \mathcal{S}_x^- - 2 \mathcal{I} + \mathcal{S}_x^+ $ and $D_y^- D_y^+ = \mathcal{S}_y^- - 2 \mathcal{I} + \mathcal{S}_y^+ $, where $\mathcal{I}$ is the identity operator. 
For discrete frequencies $x_i$ and $y_i$, we have 
\cite{tai2011fast}:
\begin{equation*}
	\begin{aligned}
&\mathcal{F} \mathcal{S}_x^\pm u(x_i,y_j)=e^{\pm \sqrt{-1}z_i} \mathcal{F} u(x_i,y_j), \\
&\mathcal{F} \mathcal{S}_y^\pm u(x_i,y_j)=e^{\pm \sqrt{-1}z_j} \mathcal{F} u(x_i,y_j),
	\end{aligned}
\end{equation*}
where 
\begin{equation*}
	z_i=\frac{2\pi}{N_1} x_i, i=1,2,\cdots,N_1, \quad 
	z_j=\frac{2\pi}{N_2} y_j, j=1,2,\cdots,N_2 .
\end{equation*}
Thus, the implicit scheme in \eqref{fft} can be rewritten as
\begin{equation}\label{alg_eq}
	(-2\tau(\operatorname{cos}(z_i)+\operatorname{cos}(z_j)-2)+1)\mathcal{F} u^{k+1}(x_i,y_j) = \mathcal{F} u^{k}(x_i,y_j).
\end{equation}
The discrete inverse Fourier transformation $\mathcal{F}^{-1}$ can be directly applied to the solution of \eqref{alg_eq} to obtain the updated $u^{k+1}$.	

The whole algorithm of our model is given in Algorithm \ref{alg:alg1}.

Until now, the effectiveness of our model is explained from the modeling principle and theory. 
The experiments evidences are presented in section \ref{sec:exp}.

\section{Experiment}\label{sec:exp}
\subsection{Data Description}
We randomly selected 400 images for training and 68 images for validation from the BSR-BSDS500 dataset.	
Noisy images were generated by adding multiplicative Gamma noise to both the training and validation sets after converting the color images to grayscale, with a noise level of $L=1, 4, 10$.	

The test data comprised two types: simulated images and real SAR images. For the simulated images, multiplicative Gamma noise with $L=1,4,10$ was added to four publicly available datasets: Set12 \cite{zeyde2012single}, McMaster \cite{zhang2011color}, RESISC45 \cite{cheng2017remote} and FloodNet Dataset \cite{rahnemoonfar2020floodnet}. 
Here, 12 images and 14 images were picked randomly in the FloodNet Dataset and RESISC45 as the representatives of their data set, because of the huge data volume. 
In addition, we selected the first ten images from the Set12 as the test images, following the procedure in \cite{chierchia2017sar}.
To prevent the emergence of idiosyncratic samples, noise was randomly added to each image 10 times. 
Three real SAR images were used for  evaluation of our method, as shown in Figure \ref{real-sar}. For simplicity, these images were denoted as SAR1 ($512\times512$), SAR2 ($512\times512$), and SAR3 ($512\times512$), which came from ICEYE \footnote{ https://www.iceye.com/downloads/datasets } and Gaofen-3 satellite \footnote{  https://github.com/AmberHen/WHU-OPT-SAR-dataset}.

\begin{figure}
	\centering
    \subfigure[SAR1]{
			\includegraphics[width=1in]{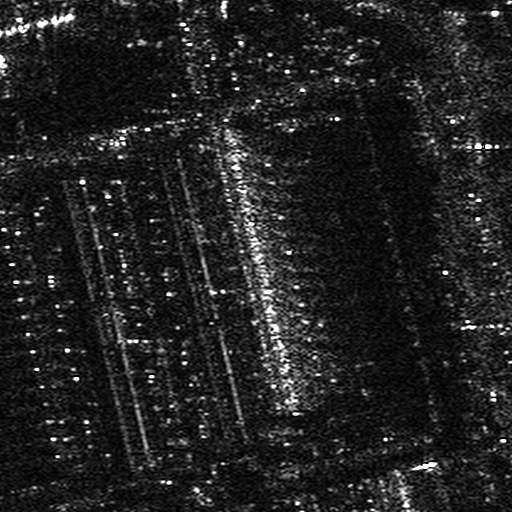}
            \label{SAR1}
    }
	\subfigure[SAR2]{
			\includegraphics[width=1in]{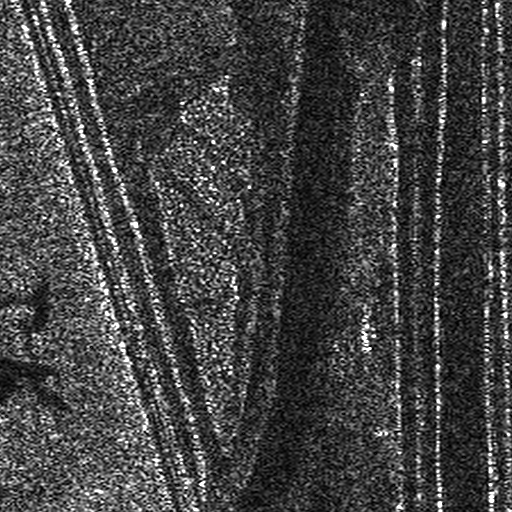}
            \label{SAR2}
	}%
	\subfigure[SAR3]{
			\includegraphics[width=1in]{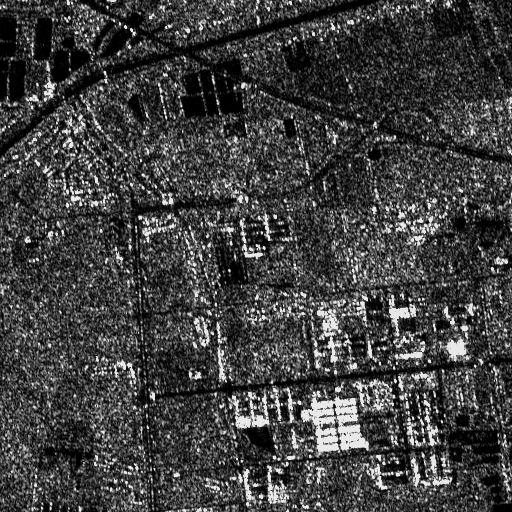}
            \label{SAR3}
	}%
	\centering
	\caption{Real SAR images.}
	\label{real-sar}
\end{figure}

\subsection{Experimental Setup and training}
In this study, our network utilized the Adam optimizer, with a minibatch size of 128 patches, and adopted the batch normalization technique. 
The model was trained for 50 epochs, with learning rates 0.001 for the first 30 epochs and reduced to 0.0001 for the subsequent 20 epochs.	
Following data augmentation, the training set was composed of $2000 \times 128$ patches, with each patch including $40 \times 40$ pixels.

During training, the time step in the diffusion equation unit and the unrolling depth were set to $\tau = 0.1$ and $K=5$ respectively. After training, the time step can be adjusted freely, while the other network parameters remain fixed.	
Throughout the training and testing phases, all code was implemented by PyTorch framework.

\subsubsection{Evaluation Index}
Distinct evaluation strategies are adopted for simulated and real SAR images.	

For simulated images, evaluation metrics include the peak signal-to-noise ratio (PSNR) and the structural similarity image measurement (SSIM).	
Here, the clean image and denoised image are denoted as $y$ and $\hat{y}$ respectively.  
\begin{gather*}
    \operatorname{PSNR}=10 \log _{10} \frac{I J 255 * 255}{\left\|\hat{y}-y\right\|_{L^2}^2},  \\
\mathrm{SSIM}=\frac{\left(2 \mu_{y} \mu_{\hat{y}}+c_1\right)\left(2 \sigma_{\hat{y} y}+c_2\right)}{\left(\mu_{\hat{y}}^2+\mu_{y}^2+c_1\right)\left(\sigma_{\hat{y}}^2+\sigma_{y}^2+c_2\right)},
\end{gather*}
where $I$ and $J$ are height and width of the image respectively; 
$\mu_{\hat{y}}$ and $\mu_{y}$ represent the mean of $\hat{y}$ and $y$; 
$\sigma_{\hat{y}}^2$ and $\sigma_{y}^2$ represent the variance of $\hat{y}$ and $y$; 
$\sigma_{\hat{y} y}$ is the
covariance between $\hat{y}$ and $y$; $c_1$ and $c_2$ are constants.

For real SAR images, in the absence of ground-truth, the evaluation of denoising cleanliness is based on the equivalent number of looks (ENL) 
\begin{equation*}
	\mathrm{ENL}=\frac{\bar{\hat{y}}^2}{\sigma_{\hat{y}}^2},
\end{equation*}
where  $\bar{\hat{y}}^2$ and $\sigma_{\hat{y}}^2$ are the mean and the variance of the restored image  respectively. 
The edge preservation degree is determined by the ratio of the standard deviation $\sigma_{\hat{y}}^2$ to the mean intensity $\mu_{\hat{y}}$ (Cx) and the ratio of average (EPD-ROA) for horizontal direction (HD) and vertical direction (VD) 
\begin{equation*}
	\mathrm{Cx}=\frac{\sigma_{\hat{y}}}{\mu_{\hat{y}}}, \quad
	\text { EPD-ROA }=\frac{\sum_{i=1}^m\left|I_{D 1}(i) / I_{D 2}(i)\right|}{\sum_{i=1}^m\left|I_{O 1}(i) / I_{O 2}(i)\right|} ,
\end{equation*}
where $I_{D 1}(i)$ and $I_{D 2}(i)$ are the adjacent pixels on the horizontal and vertical direction of denoised image. $I_{O 1}(i)$ and $I_{O 2}(i)$ are the adjacent pixels on the horizontal and vertical direction of noisy image. 
In computation, to avoid zero in the denominator, the gray value range of the noise image and the restored image is changed from $[0, 255]$ to $[1, 256]$.
In a homogeneous area, large (small) coefficient of ENL (Cx) represents great speckle removal performance. 
The closer of RPD-ROA value to 1, the more details  are protected.

\subsubsection{Comparison Method}
The experimental procedure is divided into three distinct phases: simulated image experiments, adversarial {sample} experiments, and real SAR image tests.	
We compare the proposed model against one traditional method: MuLoG-BM3D \cite{deledalle2017mulog}, and four state-of-the-art neural network (NN)-based methods: SAR-CNN \cite{chierchia2017sar}, AGSDNet \cite{thakur2022agsdnet}, SAR-RDCP \cite{shen2020sar}, and TB-SAR\cite{perera2022transformer}.	
SAR-CNN represents the earliest application of deep learning for multiplicative noise removal.	
AGSDNet integrates image gradient information, channel and spatial attention mechanisms into the network, resulting in enhanced performance.	
SAR-RDCP uses neural network priors in a variational model and unrolls its algorithm into the network, effectively enhancing despeckling.	
In contrast to the above CNN-based methods, TB-SAR is a denoising model specifically designed under the transformer framework.

\begin{table*}[htbp]\tiny 
	\centering
	\caption{The comparison of PSNR and SSIM on simulate images}\label{simlua}
    \resizebox{\linewidth}{!}{
	
	\begin{tabular}{cccccccccccccc}
	\hline
	\multicolumn{1}{c}{\multirow{2}{*}{Dataset}}             & \multicolumn{1}{c}{\multirow{2}{*}{Looks}} & \multicolumn{2}{c}{MuLoG-BM3D}    & \multicolumn{2}{c}{SAR-CNN}  & \multicolumn{2}{c}{AGSDNet} & \multicolumn{2}{c}{SAR-RDCP} & \multicolumn{2}{c}{TB-SAR}   & \multicolumn{2}{c}{Ours}    \\ \cline{3-14} 
	\multicolumn{1}{c}{}                                     & \multicolumn{1}{c}{}                       & \multicolumn{1}{c}{PSNR}   & \multicolumn{1}{c}{SSIM} & \multicolumn{1}{c}{PSNR}     & \multicolumn{1}{c}{SSIM} & \multicolumn{1}{c}{PSNR}  & \multicolumn{1}{c}{SSIM} & \multicolumn{1}{c}{PSNR}  & \multicolumn{1}{c}{SSIM} & \multicolumn{1}{c}{PSNR} & \multicolumn{1}{c}{SSIM} & \multicolumn{1}{c}{PSNR}  & \multicolumn{1}{c}{SSIM} \\ 
	\hline
	\multirow{3}{*}{Set12} & L=1	&21.52	&0.75	&24.67	&0.71	&25.64	&0.74	&23.36	&0.68	&24.64	&0.71	&\textbf{25.78}	&\textbf{0.77}	\\
									
	&L=4&26.65&0.82&29.04&\textbf{0.85}&28.72&0.83 &28.44 &0.82 &27.03&0.77&\textbf{29.17}&\textbf{0.85}\\
													& L=10&29.02  
													&0.87&30.57&0.85&30.93&0.87 &30.57 &0.87 &28.43& 0.81&\textbf{31.15}&\textbf{0.89}\\
	\hline
	\multirow{3}{*}{McMaster} & L=1&22.60&0.71&25.59&0.76&26.16&0.77 &22.12 &0.70 &25.12&0.72&\textbf{26.33}&\textbf{0.80}\\
															   & L=4&26.77&0.80&29.72&0.88&29.40&0.87 &29.16 &0.86 &27.36&0.80&\textbf{29.84}&\textbf{0.89}\\
															   & L=10&29.11
															   &0.86&31.64&0.90&31.79& 0.91&31.50 & 0.91&28.85&0.83&\textbf{32.25}&\textbf{0.93}\\
	\hline										\multirow{3}{*}{FloodNet} & L=1&22.56&0.71&25.21& 0.61&27.05&0.66 &22.86 &0.61 &26.46&0.63&\textbf{27.14}&\textbf{0.69}\\
															   & L=4 &27.45&\textbf{0.79}&29.49&0.77&29.41&0.75 &29.28 &0.75 &28.07&0.69&\textbf{29.66}&0.78\\
															   & L=10&29.16&\textbf{0.84}&30.79&0.80&31.22& 0.81&31.12 &0.81 &29.15&0.73&\textbf{31.47}&0.83\\
	\hline										\multirow{3}{*}{NWPU} & L=1&21.26&0.64&23.70&0.67&24.41&0.70&21.28&0.65&23.59&0.65&\textbf{24.57}&\textbf{0.72}\\
															   & L=4 &25.41&0.74&27.61&\textbf{0.83}&27.35&0.81&27.09&0.81&25.82&0.74&\textbf{27.69}&\textbf{0.83}\\
															   & L=10&27.33&0.79&29.36&0.87&29.52&0.87&29.29&0.87&27.24&0.78&\textbf{29.84}&\textbf{0.88}\\
\hline
	\end{tabular}}
	\end{table*}
\vspace{-5mm}

During the comparative experiment, the parameter settings in MuLoG-BM3D,  SAR-CNN, AGSDNet, TB-SAR refer to the corresponding references. In SAR-RDCP, the trade-off parameter $\lambda$ is set to 0.48, and the other parameters are same as its paper. 

\subsection{Results on simulated images}
In denoising experiments with simulated SAR images, both visual evaluation and quantitative comparison are essential.	
In this experiment,  all the test images are selected from Set12, McMaster,  FloodNet Dataset and RESISC45, and corrupted by multiplicative Gamma noise with noise level $L=1,4,10$.

Table \ref{simlua} present the average PSNR and SSIM for each test image, with the best values highlighted in bold.	
It is evident that the traditional method MuLoG-BM3D exhibits the lowest PSNR and SSIM in comparison to NN-based methods, further demonstrating the significant advantages of neural networks.	
As shown in Table \ref{simlua}, our model provides superior detail preservation (SSIM) and enhanced despeckling capability (PSNR) compared to other methods, regardless of the noise level.	
SARCNN and AGSDNet exhibit comparable performance to our model when the noisy level $L=4$ and 10, but AGSDNet shows better denoising results when $L=1$. 
For the FloodNet dataset, the BM3D method shows higher SSIM when $L=4$ and 10, only slightly higher than our method.

From a visual perspective, the denoising capability of a model is generally {evaluated based on its performance} on smooth backgrounds, {edge preservation, and texture retention.}
To address these three aspects, we select representative images from the test sets to illustrate the key characteristics of our model. {All} images are selected under the noise level $L=1$, which is the worst situation in real life to test the capability of models to cope with the most complex environments.

\begin{figure*}[htbp]
    \centering
    \begin{tabular}{@{}c@{~}c@{~}c@{~}c@{~}c@{~}c@{~}c@{}}

		\includegraphics[width=.13\textwidth]{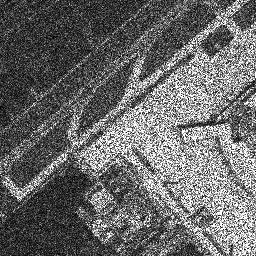} &
		\includegraphics[width=.13\textwidth]{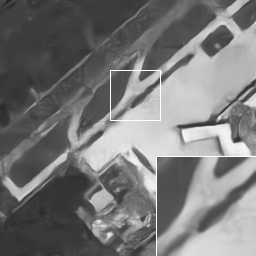} &
        \includegraphics[width=.13\textwidth]{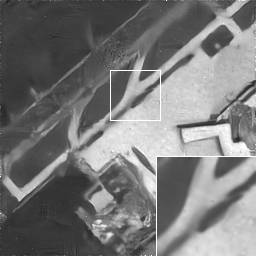} &
        \includegraphics[width=.13\textwidth]{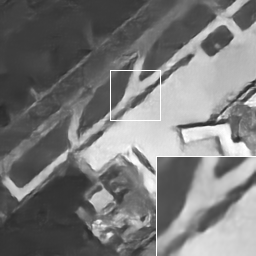} &
        \includegraphics[width=.13\textwidth]{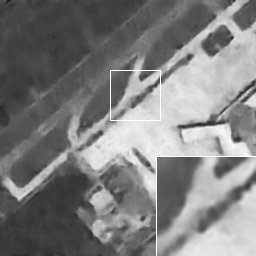} &
        \includegraphics[width=.13\textwidth]{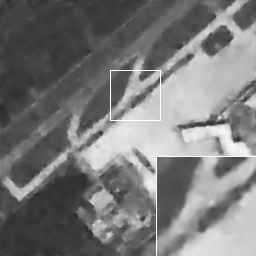} &
		\includegraphics[width=.13\textwidth]{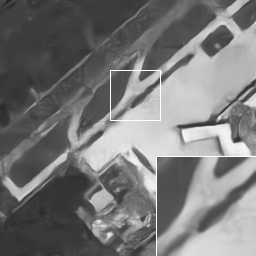} \\

		\includegraphics[width=.13\textwidth]{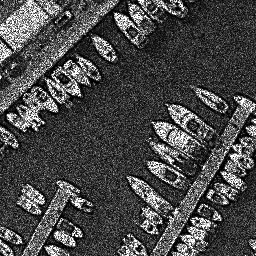} &
		\includegraphics[width=.13\textwidth]{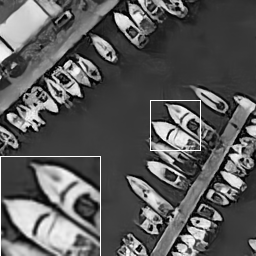} &
        \includegraphics[width=.13\textwidth]{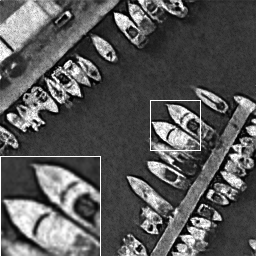} &
        \includegraphics[width=.13\textwidth]{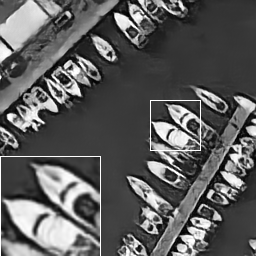} &
        \includegraphics[width=.13\textwidth]{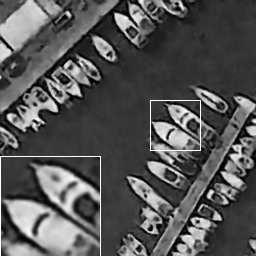} &
        \includegraphics[width=.13\textwidth]{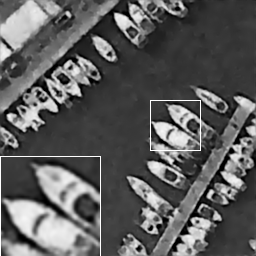} &
		\includegraphics[width=.13\textwidth]{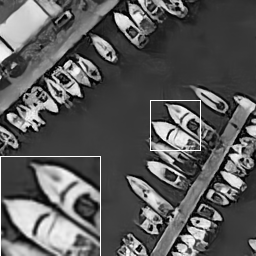} \\

        \includegraphics[width=.13\textwidth]{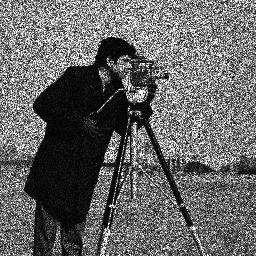} &
		\includegraphics[width=.13\textwidth]{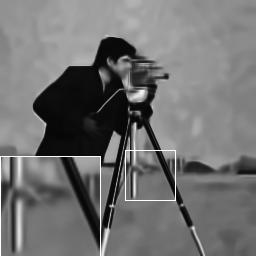} &
        \includegraphics[width=.13\textwidth]{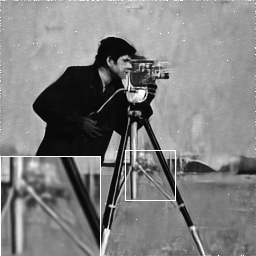} &
        \includegraphics[width=.13\textwidth]{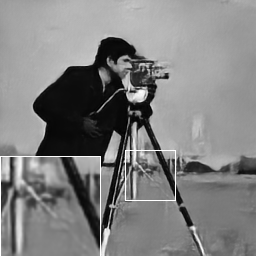} &
        \includegraphics[width=.13\textwidth]{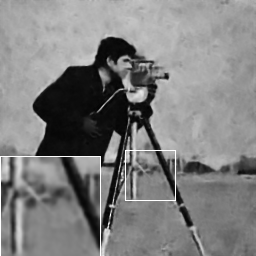} &
        \includegraphics[width=.13\textwidth]{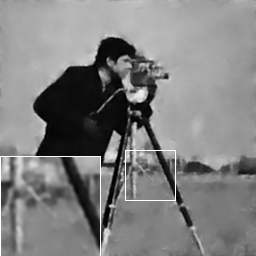} &
		\includegraphics[width=.13\textwidth]{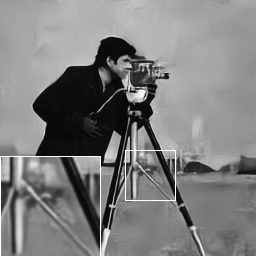} \\

		\includegraphics[width=.13\textwidth]{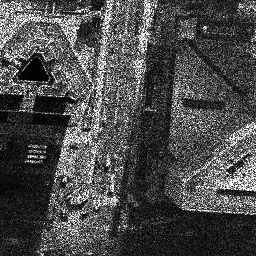} &
		\includegraphics[width=.13\textwidth]{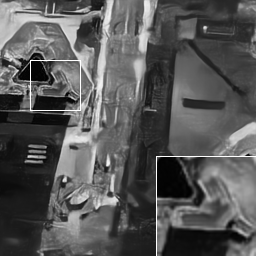} &
        \includegraphics[width=.13\textwidth]{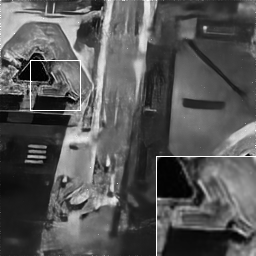} &
        \includegraphics[width=.13\textwidth]{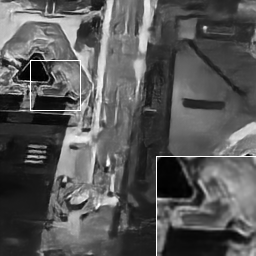} &
        \includegraphics[width=.13\textwidth]{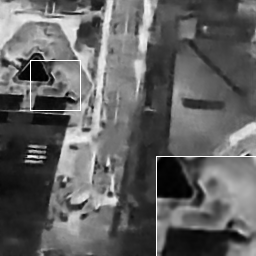} &
        \includegraphics[width=.13\textwidth]{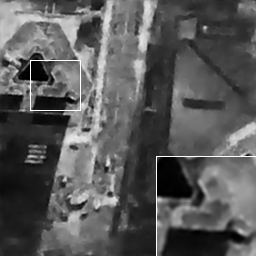} &
		\includegraphics[width=.13\textwidth]{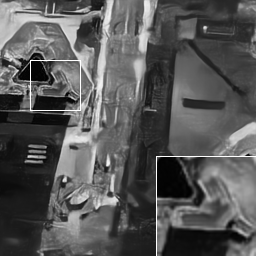} \\

		\includegraphics[width=.13\textwidth]{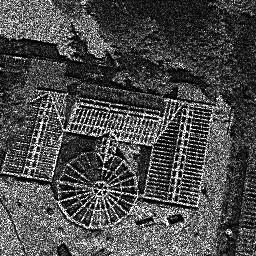} &
		\includegraphics[width=.13\textwidth]{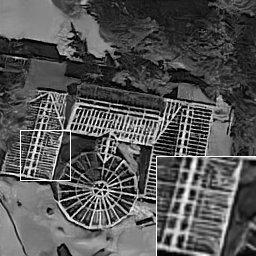} &
        \includegraphics[width=.13\textwidth]{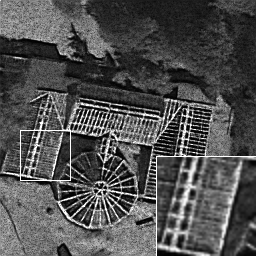} &
        \includegraphics[width=.13\textwidth]{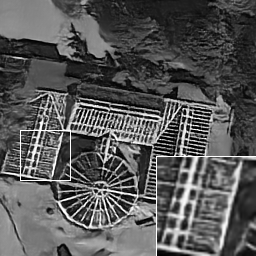} &
        \includegraphics[width=.13\textwidth]{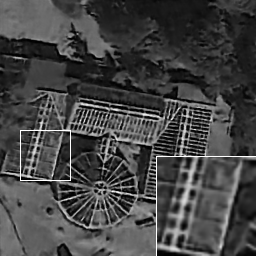} &
        \includegraphics[width=.13\textwidth]{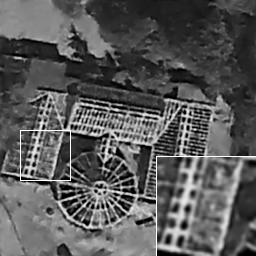} &
		\includegraphics[width=.13\textwidth]{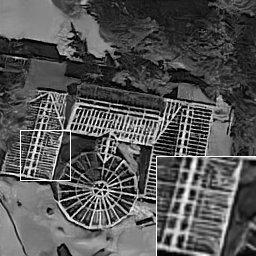} \\

		\includegraphics[width=.13\textwidth]{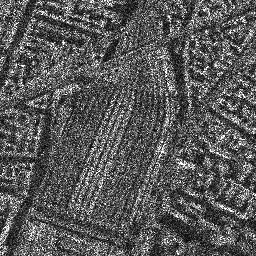} &
		\includegraphics[width=.13\textwidth]{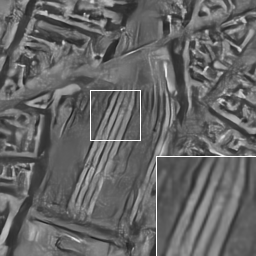} &
        \includegraphics[width=.13\textwidth]{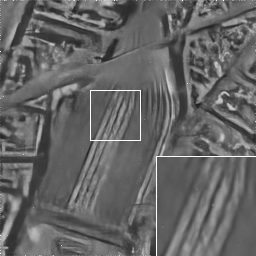} &
        \includegraphics[width=.13\textwidth]{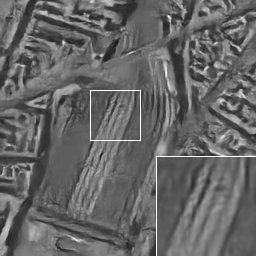} &
        \includegraphics[width=.13\textwidth]{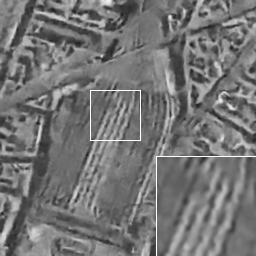} &
        \includegraphics[width=.13\textwidth]{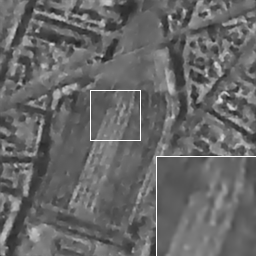} &
		\includegraphics[width=.13\textwidth]{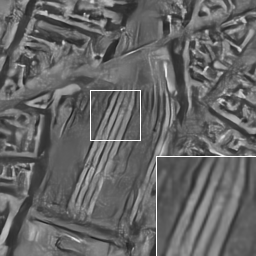} \\

		\footnotesize{\tiny Nosiy} & 
		\footnotesize{\tiny MuLoG-BM3D} & 
		\footnotesize{\tiny SAR-CNN}  & 
		\footnotesize{\tiny AGSDNet}  &
		\footnotesize{\tiny SAR-RDCP}  &
		\footnotesize{\tiny TB-SAR}  & 
		\footnotesize{\tiny Ours} \\

      \end{tabular}
      \caption{Restoration results for images with large homogenous areas with noise level $L = 1$.}
      \label{large_area=1}
\end{figure*}

The first two images selected in Figure \ref{large_area=1} share two common characteristics: uniform grayscale in large areas and high-contrast edges.	
In the restoration results of these two images, all methods demonstrate excellent preservation of strong (coarse) edges in the restored images.	
Furthermore, our model, AGSDNet, and MuLoG-BM3D effectively remove background noise, with our model and AGSDNet producing sharper edges in the restored results.

The third and fourth images in Figure \ref{large_area=1} contain rich weak edge  informations, such as the camera tripods in the third image, and triangular building on the roof in the fourth image, these are overly smoothed by the MuLoG-BM3D, SAR-RDCP, and TBSAR methods.	
Additionally, for the  camera tripods in the third image, our model and SAR-CNN produce similar results, outperforming AGSDNet and other methods.

\begin{table}[htbp]
    \centering
	\caption{The running time and numbers of parameter for each network }\label{time_para}
    \resizebox{\linewidth}{!}{
    \begin{tabular}{ccccccccccc}
    \hline
    \multirow{2}{*}{dataset}  & \multicolumn{2}{c}{SAR-CNN}  & \multicolumn{2}{c}{AGSDNet} & \multicolumn{2}{c}{SAR-RDCP} & \multicolumn{2}{c}{TB-SAR}   & \multicolumn{2}{c}{Ours} \\ \cline{2-11} 
                             & Time      & Parameters     & Time      & Parameters      & Time     & Parameters    & Time    & Parameters   & Time     & Parameters    \\ \hline
    set12                    &      8.429        &    556097            &   4.525        &   652916              &   7.753       &      272195         &   5.437      &   24099201           &      9.028    &    560965           \\ \hline
    \end{tabular}}
    \end{table}

\begin{table}[htbp] \tiny
	\centering
	\caption{The PSNR on the adversarial samples generated by Set12 }\label{Compare_set12_ADV}
	\begin{tabular}{ccccccc}   
	\toprule
	\text{Dataset} & \text{SAR-CNN} & \text{AGSDNet} & SAR-RDCP & TB-SAR&Ours\\ 
	\midrule
	
	Set12 		&23.73	&24.15	&22.31		&23.85	&24.19\\ 

	 \bottomrule
	\end{tabular}
	\end{table}

The restoration of texture is shown in the last two images in Figure \ref{large_area=1}.
Our model achieves similar results to SAR-CNN on the last image, outperforming the other compared models.	
Our model, SAR-CNN and AGSDNet have similar restoration results on the pattern on the roof of church.

For a more comprehensive comparison with baseline models, the testing runtime on Set12 and the parameters count of each network are summarized in Table \ref{time_para}.	
As shown in Table \ref{time_para}, our model has fewer parameters than TB-SAR and AGSDNet, but {a} slightly longer runtime compared to other methods.	
The slightly longer runtime of our model stems from the matrix division required to solve the implicit scheme of the heat equation, whose computation time increases with image size.	
However, our model's runtime is only marginally slower than other NN-based methods, remaining within an acceptable range.

In summary, our model demonstrates superior performance on simulated images compared to the other models both on evaluation index and visual effect.

\subsection{Results on adversarial simples }\label{subsec:result_adv}

The performance of neural networks on real data often lags behind their performance on simulated datasets due to their instability.
In this section, adversarial attacks are employed to simulate severe perturbations between real-world and simulated data, serving as a benchmark to evaluate the robustness of the proposed network.

As mentioned earlier, the adversarial samples generated by denoising-PGD exhibit strong transferability across NN-based methods, allowing a single adversarial sample to effectively attack all such models \cite{ning2023evaluating}.	
In subsection \ref{subsec:difff}, the principle that our model can counter such adversarial attacks had been explained, and in this section, 
we verify its effectiveness through experimental evidence.	 
In this experiment, the images in Set12 are selected as the attack target, $L^\infty$-denoising-PGD is chosen as the attack strategy, and  all adversarial samples were generated from a 10-layer SAR-CNN.	

Table \ref{Compare_set12_ADV} lists the  PSNR values of all the methods on the adversarial {samples}. Because of the adding of new adversarial noise to the original noisy image, the decline of evaluation index is inevitable. From quantitative index, our model is better than ours, but not too much. 
The reason is that adversarial attacks turn details into high-frequency oscillations, which has already lost information and  the diffusion equation cannot generate the lost information.

The primary advantages of our model are reflected in visual performance.	
Our model effectively suppresses high-frequency oscillations within the neural network, thereby enhancing restoration quality.	

Figure \ref{adversarial_L=1} illustrates the visual performance of NN-based models on adversarial samples and non adversarial images.  	
The first and second columns depict the adversarial samples and its differences with the original noisy images, respectively.	
The third to sixth columns show the results generated by different NN-based approaches.	

As shown in Figure \ref{adversarial_L=1}, contrary to the results for non adversarial images, all restored adversarial {samples} from the compared methods exhibit oscillations in the background and main objects, which should not occur.	
The restoration results on adversarial {samples} produced by SAR-CNN and AGSDNet display high-frequency oscillations in the background and detailed regions, negatively impacting object recognition.	
The restoration results on adversarial {samples} generated by TB-SAR have blurred edges, which reduces the overall restoration quality.	
Due to the cumulative regularization, the images restored by our model exhibit superior denoising effects and sharper edges compared to those restored by other models.	
Our model removes unnecessary high-frequency oscillations in edges and backgrounds while preserving image details. 

\begin{figure}[H]
    \centering
    \begin{tabular}{@{}c@{~}c@{~}c@{~}c@{~}c@{~}c@{~}c@{}}

        \raisebox{5ex}{\makecell[c]{{\footnotesize{a*}}}}&
        \includegraphics[width=.15\textwidth]{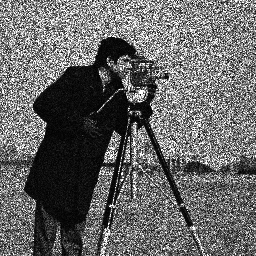} &
        \includegraphics[width=.15\textwidth]{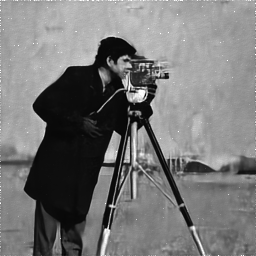} &
        \includegraphics[width=.15\textwidth]{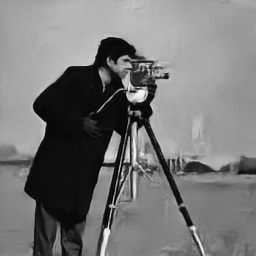} &
        \includegraphics[width=.15\textwidth]{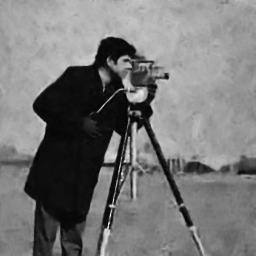} &
        \includegraphics[width=.15\textwidth]{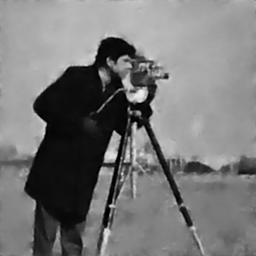} &
		\includegraphics[width=.15\textwidth]{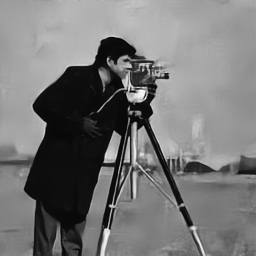} \\

        \raisebox{5ex}{\makecell[c]{{\footnotesize{b*}}}}&
        \includegraphics[width=.15\textwidth]{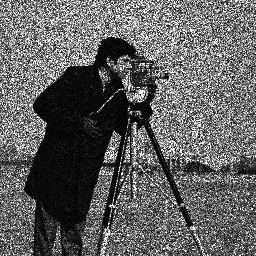} &
        \includegraphics[width=.15\textwidth]{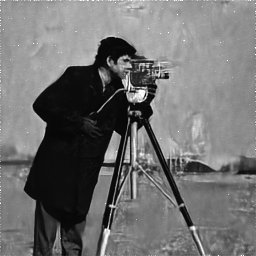} &
        \includegraphics[width=.15\textwidth]{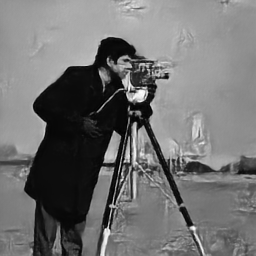} &
        \includegraphics[width=.15\textwidth]{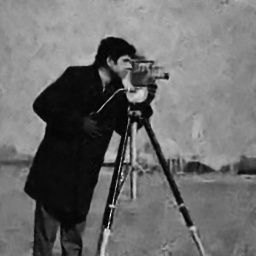} &
        \includegraphics[width=.15\textwidth]{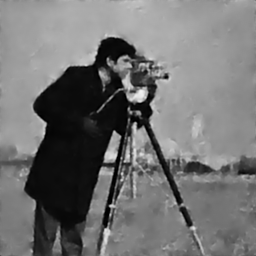} &
		\includegraphics[width=.15\textwidth]{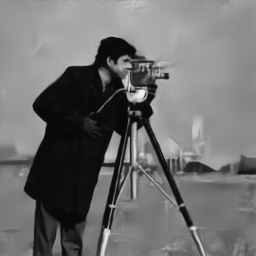} \\

        \raisebox{5ex}{\makecell[c]{{\footnotesize{a*}}}}&
        \includegraphics[width=.15\textwidth]{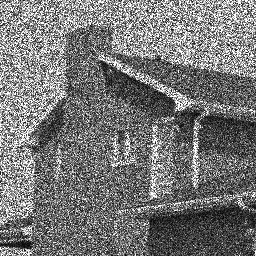} &
        \includegraphics[width=.15\textwidth]{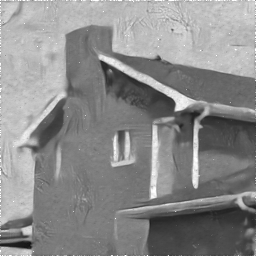} &
        \includegraphics[width=.15\textwidth]{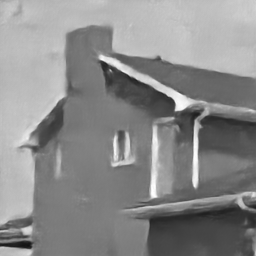} &
        \includegraphics[width=.15\textwidth]{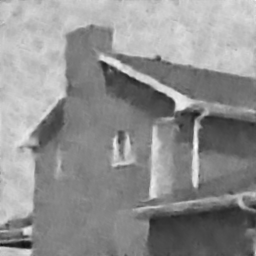} &
        \includegraphics[width=.15\textwidth]{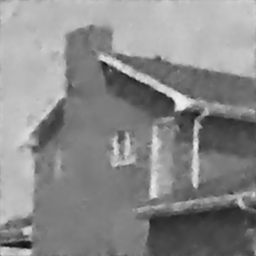} &
		\includegraphics[width=.15\textwidth]{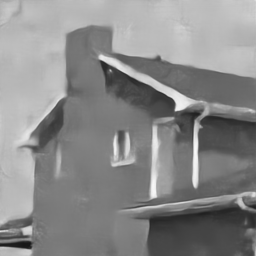} \\

		\raisebox{5ex}{\makecell[c]{{\footnotesize{b*}}}}&
        \includegraphics[width=.15\textwidth]{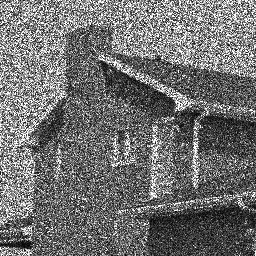} &
        \includegraphics[width=.15\textwidth]{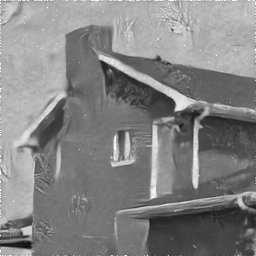} &
        \includegraphics[width=.15\textwidth]{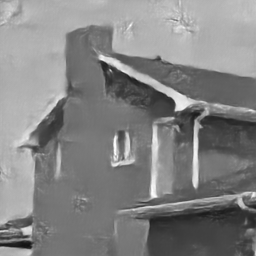} &
        \includegraphics[width=.15\textwidth]{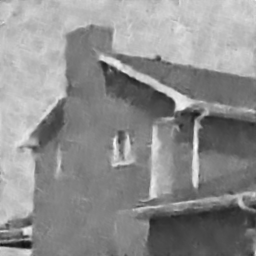} &
        \includegraphics[width=.15\textwidth]{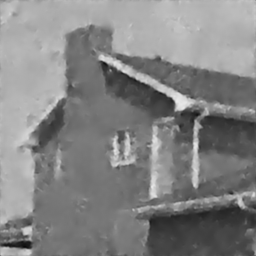} &
		\includegraphics[width=.15\textwidth]{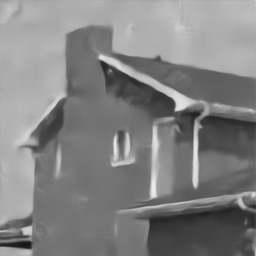} \\

        \raisebox{5ex}{\makecell[c]{{\footnotesize{a*}}}}&
        \includegraphics[width=.15\textwidth]{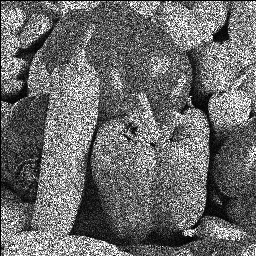} &
        \includegraphics[width=.15\textwidth]{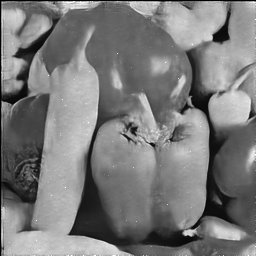} &
        \includegraphics[width=.15\textwidth]{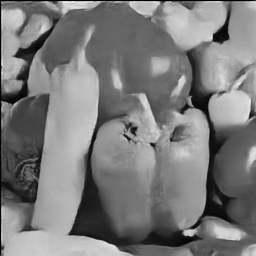} &
        \includegraphics[width=.15\textwidth]{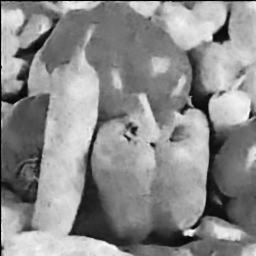} &
        \includegraphics[width=.15\textwidth]{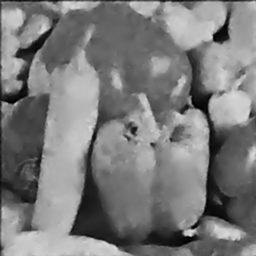} &
		\includegraphics[width=.15\textwidth]{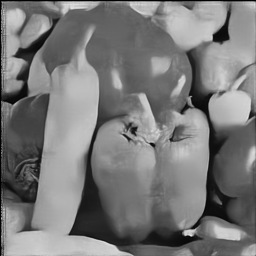} \\

		\raisebox{5ex}{\makecell[c]{{\footnotesize{b*}}}}&
        \includegraphics[width=.15\textwidth]{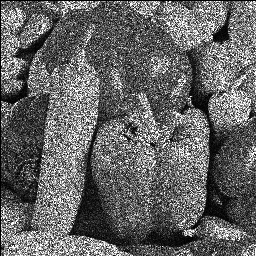} &
        \includegraphics[width=.15\textwidth]{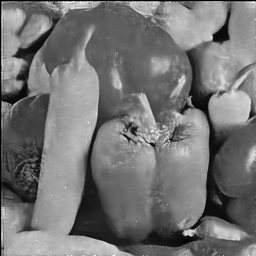} &
        \includegraphics[width=.15\textwidth]{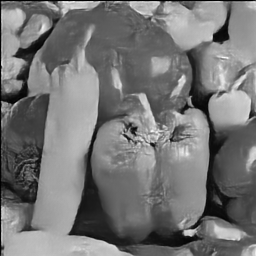} &
        \includegraphics[width=.15\textwidth]{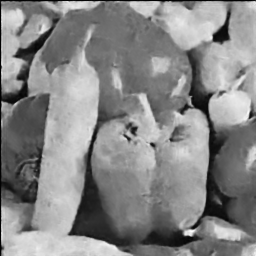} &
        \includegraphics[width=.15\textwidth]{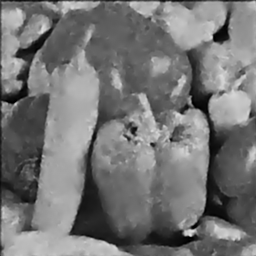} &
		\includegraphics[width=.15\textwidth]{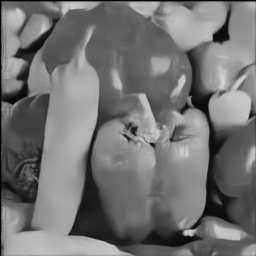} \\

        \raisebox{5ex}{\makecell[c]{{\footnotesize{a*}}}}&
        \includegraphics[width=.15\textwidth]{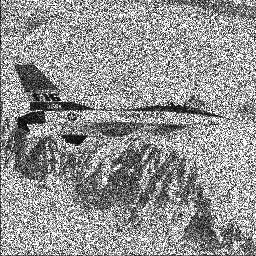} &
        \includegraphics[width=.15\textwidth]{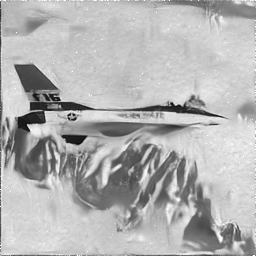} &
        \includegraphics[width=.15\textwidth]{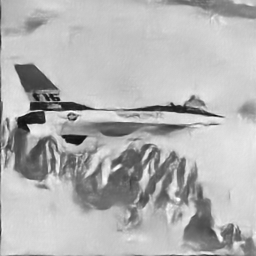} &
        \includegraphics[width=.15\textwidth]{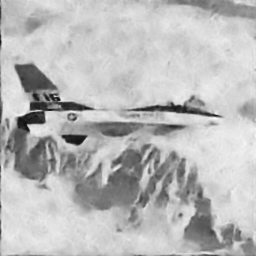} &
        \includegraphics[width=.15\textwidth]{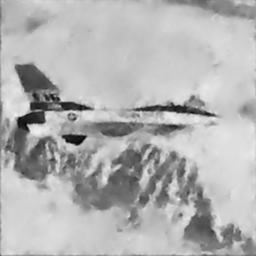} &
		\includegraphics[width=.15\textwidth]{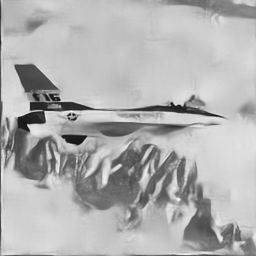} \\

		\raisebox{5ex}{\makecell[c]{{\footnotesize{b*}}}}&
        \includegraphics[width=.15\textwidth]{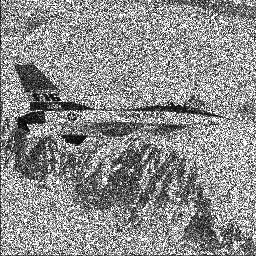} &
        \includegraphics[width=.15\textwidth]{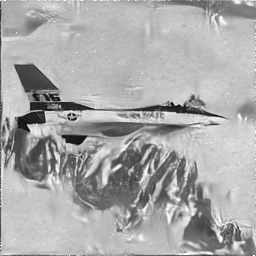} &
        \includegraphics[width=.15\textwidth]{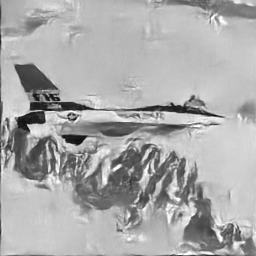} &
        \includegraphics[width=.15\textwidth]{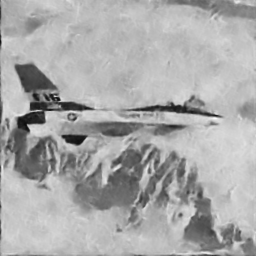} &
        \includegraphics[width=.15\textwidth]{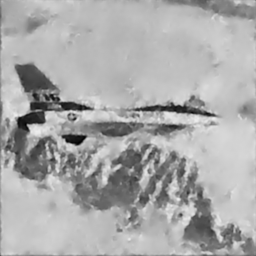} &
		\includegraphics[width=.15\textwidth]{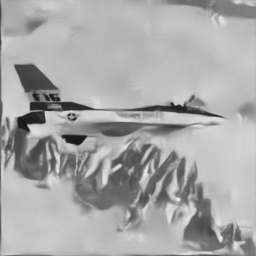} \\

        \footnotesize{ } &
		\footnotesize{\tiny Nosiy} & 
		\footnotesize{\tiny SAR-CNN}  & 
		\footnotesize{\tiny AGSDNet}  &
		\footnotesize{\tiny SAR-RDCP}  &
		\footnotesize{\tiny TB-SAR}  & 
		\footnotesize{\tiny Ours} \\
		
      \end{tabular}
      \caption{Restoration results for non adversarial images and adversarial {samples} based on noisy images with $L=1$. {Here, a* and b* represent the non adversarial restoration and adversarial restoration, respectively.}}
      \label{adversarial_L=1}
\end{figure}

The adversarial experiments demonstrate the  stability of our model, providing strong evidence for its underlying principles and theoretical framework. Consequently, our model is well-suited for the complex real-world scenarios.

\subsection{The generalization of our model}
Since previous tests are conducted on adversarial samples generated by the $L^\infty$ attack strategy, in this section we evaluate the performance of our model on adversarial samples generated by the $L^2$ attack strategy.
As before, the 10-layer SAR-CNN is selected as the target model, and the adversarial samples are generated from the noisy images using the $L^2$ attack method. 
We then test our model on these adversarial samples, and the results are shown in Figure \ref{L2_adv_test}.

\begin{figure*}[htbp]
    \centering
    \begin{tabular}{@{}c@{~}c@{~}c@{~}c@{~}c@{~}c@{}}

		\includegraphics[width=.15\textwidth]{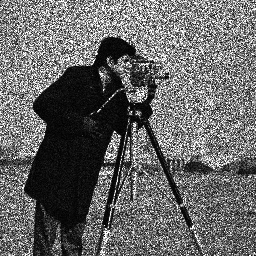} &
		\includegraphics[width=.15\textwidth]{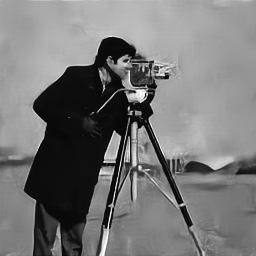} &
        \includegraphics[width=.15\textwidth]{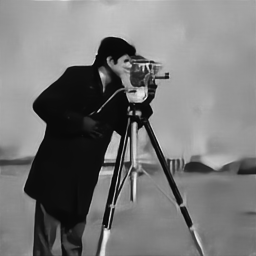} &
        \includegraphics[width=.15\textwidth]{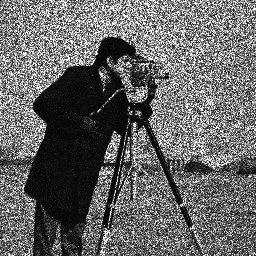} &
        \includegraphics[width=.15\textwidth]{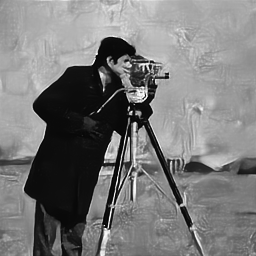} &
        \includegraphics[width=.15\textwidth]{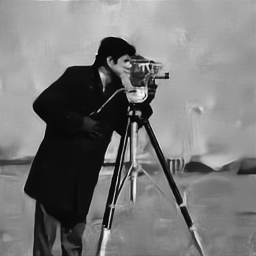} \\

		\includegraphics[width=.15\textwidth]{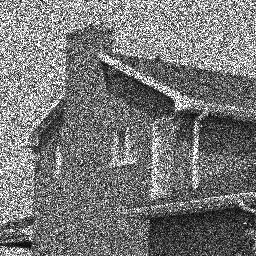} &
		\includegraphics[width=.15\textwidth]{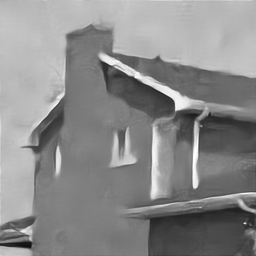} &
        \includegraphics[width=.15\textwidth]{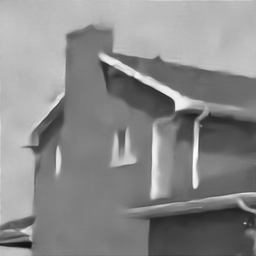} &
        \includegraphics[width=.15\textwidth]{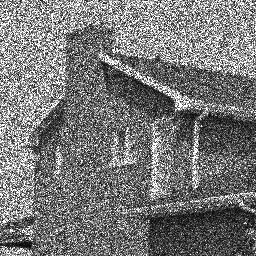} &
        \includegraphics[width=.15\textwidth]{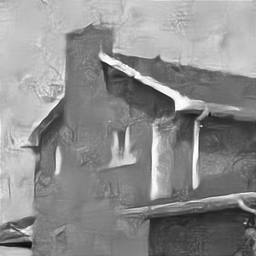} &
        \includegraphics[width=.15\textwidth]{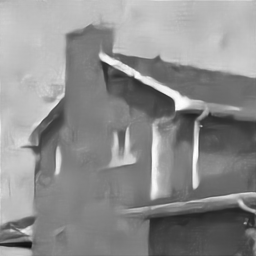} \\
        
        \includegraphics[width=.15\textwidth]{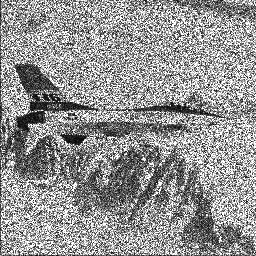} &
		\includegraphics[width=.15\textwidth]{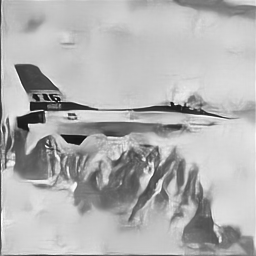} &
        \includegraphics[width=.15\textwidth]{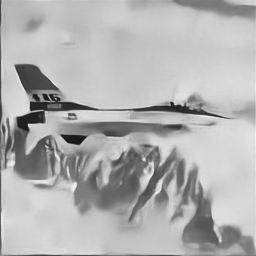} &
        \includegraphics[width=.15\textwidth]{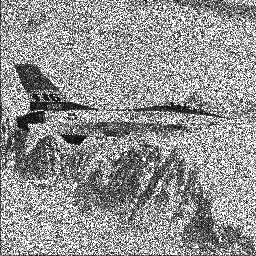} &
        \includegraphics[width=.15\textwidth]{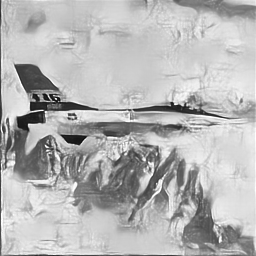} &
        \includegraphics[width=.15\textwidth]{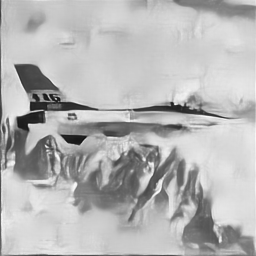} \\

        \footnotesize{\tiny Nosiy} & 
		\footnotesize{\tiny $\tau$=0.1 on } & 
		\footnotesize{\tiny $\tau$=0.13 on }  & 
		\footnotesize{\tiny adversarial}  &
		\footnotesize{\tiny $\tau$=0.1 on}  &
		\footnotesize{\tiny $\tau$=0.13 on } \\[-5pt]
		\footnotesize{\tiny image} & 
		\footnotesize{\tiny noisy image} & 
		\footnotesize{\tiny noisy image}  & 
		\footnotesize{\tiny {sample}}  &
		\footnotesize{\tiny adversarial {sample}}  &
		\footnotesize{\tiny adversarial {sample}} \\
		
      \end{tabular}
      \caption{Restoration results for non adversarial images and adversarial {samples} generated by $L^2$-denoising-PGD method. }
      \label{L2_adv_test}
\end{figure*}

As shown in Figure \ref{L2_adv_test}, adjusting the time step $\tau$ can still control the smoothness of the denoised results.
This indicates that our model also resists the $L^2$-denoising PGD attack. 
Therefore, our model exhibits some generalization capability against network adversarial attacks.

\subsection{Ablation study}\label{sec:tau_impact}
To visually assess the effect of the time step on smoothness, some experiments are designed. 

We first demonstrate how the time step controls the smoothness of outputs for both {non-adversarial images and adversarial samples}.
After training the network with $\tau=0.1$, and the unrolling depth $K=5$,  an noisy image and its corresponding adversarial  image  generated by SAR-CNN are selected as inputs.  Figures \ref{tau_impacet} and \ref{tau_impacet_no_adv} show their restoration results at different time steps, respectively.	
Increasing the time step from 0.06 to 0.20 leads to a significant increase in smoothness.

\begin{figure*}[htbp]
        \centering
    \begin{tabular}{@{}c@{~~}c@{~~}c@{~~}c@{}}
		\includegraphics[width=.2\textwidth]{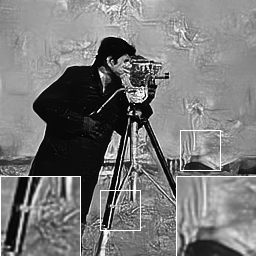} &
		\includegraphics[width=.2\textwidth]{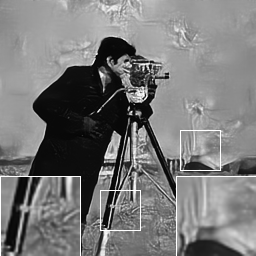} &
        \includegraphics[width=.2\textwidth]{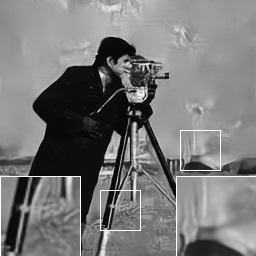} &
		\includegraphics[width=.2\textwidth]{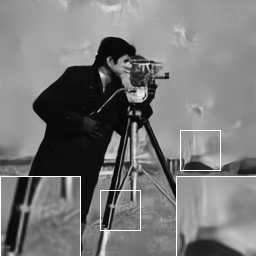} \\
        
		\footnotesize{\tiny $\tau=0.06$} & 
		\footnotesize{\tiny $\tau=0.08$} & 
		\footnotesize{\tiny $\tau=0.10$}  & 
		\footnotesize{\tiny $\tau=0.12$} \\
        \footnotesize{\tiny PSNR=21.35, SSIM=0.63 } & 
		\footnotesize{\tiny PSNR=23.43, SSIM=0.73} & 
		\footnotesize{\tiny PSNR=24.44, SSIM=0.78}  & 
		\footnotesize{\tiny PSNR=24.61, SSIM=0.79} \\

        \includegraphics[width=.2\textwidth]{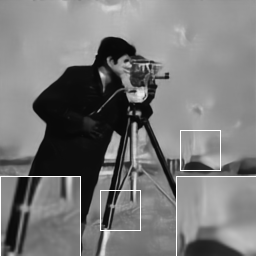} &
		\includegraphics[width=.2\textwidth]{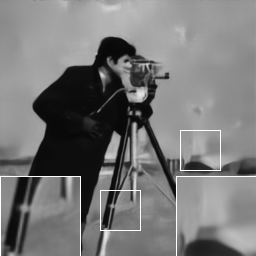} &
        \includegraphics[width=.2\textwidth]{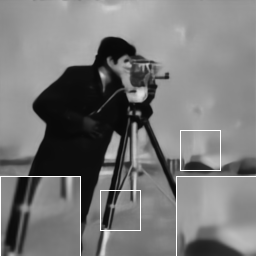} &
		\includegraphics[width=.2\textwidth]{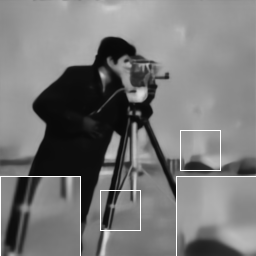} \\
		\footnotesize{\tiny $\tau=0.14$} & 
		\footnotesize{\tiny $\tau=0.16$} & 
		\footnotesize{\tiny $\tau=0.18$}  & 
		\footnotesize{\tiny $\tau=0.20$} \\
        \footnotesize{\tiny PSNR=24.34, SSIM=0.79 } & 
		\footnotesize{\tiny PSNR=23.94, SSIM=0.78} & 
		\footnotesize{\tiny PSNR=23.53, SSIM=0.77}  & 
		\footnotesize{\tiny PSNR=23.17, SSIM=0.76} \\

      \end{tabular}
      \caption{The results of our model on the adversarial {sample} with variance of $\tau$.}
	\label{tau_impacet}
\end{figure*}
\begin{figure*}[htbp]
        \centering
    \begin{tabular}{@{}c@{~~}c@{~~}c@{~~}c@{}}
		\includegraphics[width=.2\textwidth]{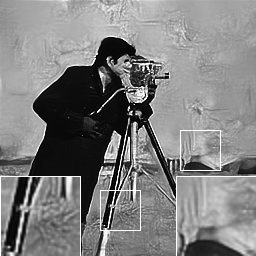} &
		\includegraphics[width=.2\textwidth]{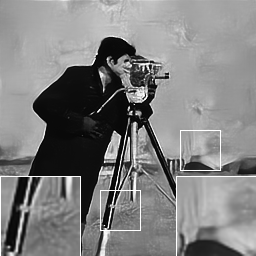} &
        \includegraphics[width=.2\textwidth]{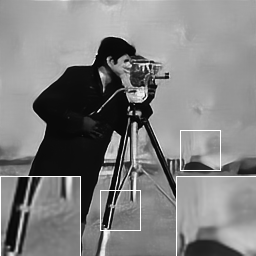} &
		\includegraphics[width=.2\textwidth]{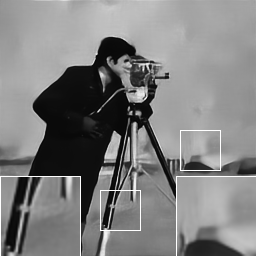} \\
        
		\footnotesize{\tiny $\tau=0.06$} & 
		\footnotesize{\tiny $\tau=0.08$} & 
		\footnotesize{\tiny $\tau=0.10$}  & 
		\footnotesize{\tiny $\tau=0.12$} \\
        \footnotesize{\tiny PSNR=22.29, SSIM=0.68 } & 
		\footnotesize{\tiny PSNR=24.68, SSIM=0.77} & 
		\footnotesize{\tiny PSNR=25.88, SSIM=0.80}  & 
		\footnotesize{\tiny PSNR=26.15, SSIM=0.80} \\

        \includegraphics[width=.2\textwidth]{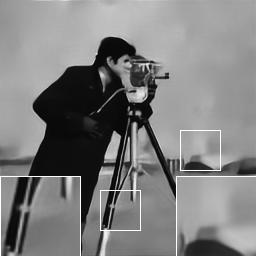} &
		\includegraphics[width=.2\textwidth]{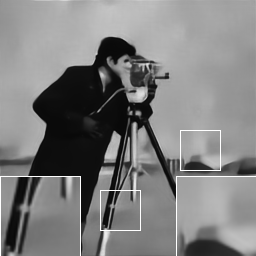} &
        \includegraphics[width=.2\textwidth]{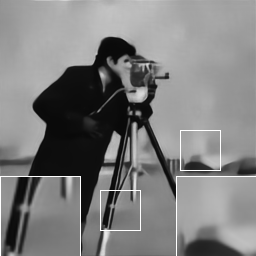} &
		\includegraphics[width=.2\textwidth]{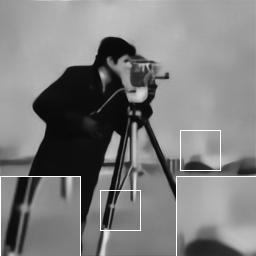} \\
		\footnotesize{\tiny $\tau=0.14$} & 
		\footnotesize{\tiny $\tau=0.16$} & 
		\footnotesize{\tiny $\tau=0.18$}  & 
		\footnotesize{\tiny $\tau=0.20$} \\
        \footnotesize{\tiny PSNR=25.93, SSIM=0.80 } & 
		\footnotesize{\tiny PSNR=25.56, SSIM=0.79} & 
		\footnotesize{\tiny PSNR=25.16, SSIM=0.78}  & 
		\footnotesize{\tiny PSNR=24.78, SSIM=0.77} \\

      \end{tabular}
      \caption{The results of our model on the non adversarial image with variance of $\tau$.}
	\label{tau_impacet_no_adv}
\end{figure*}

As shown in Figure \ref{tau_impacet}-\ref{tau_impacet_no_adv}, smaller time steps (e.g., $\tau=0.06$, $0.08$, $\tau=0.1$) preserve more details and better edge retention (e.g., the `Tower' structure) but show weaker denoising performance in background areas. 
When $\tau$ less than the preset training value ($\tau=0.1$), restored images exhibit significant oscillations, substantially degrading image quality and reducing PSNR or SSIM metrics.
Larger time steps than the preset time step $\tau=0.1$  result in greater smoothness. 
But this comes with blurring of edges and loss of details, such as the `Tower' in the image.

Notably, when increasing the time step beyond the preset value, both PSNR and SSIM initially increase {and} then decrease.	
To visualize this trend, Figure \ref{psnr_ssim_tau} presents two line plots with $\tau$ values on the horizontal axis and PSNR or SSIM metrics on the vertical axis.
Here, $\tau=0.1$ is the preset parameter {used} in network training, yet, the PSNR value {for} $\tau=0.12$ is higher. 
{This}  indicates that an appropriate time step  can improve the restoration quality. 
These experiments validate our time-step-based smoothness control strategy.

\begin{figure*}[htbp]
        \centering
    \begin{tabular}{@{}c@{~~}c@{}}
		\includegraphics[width=.48\textwidth]{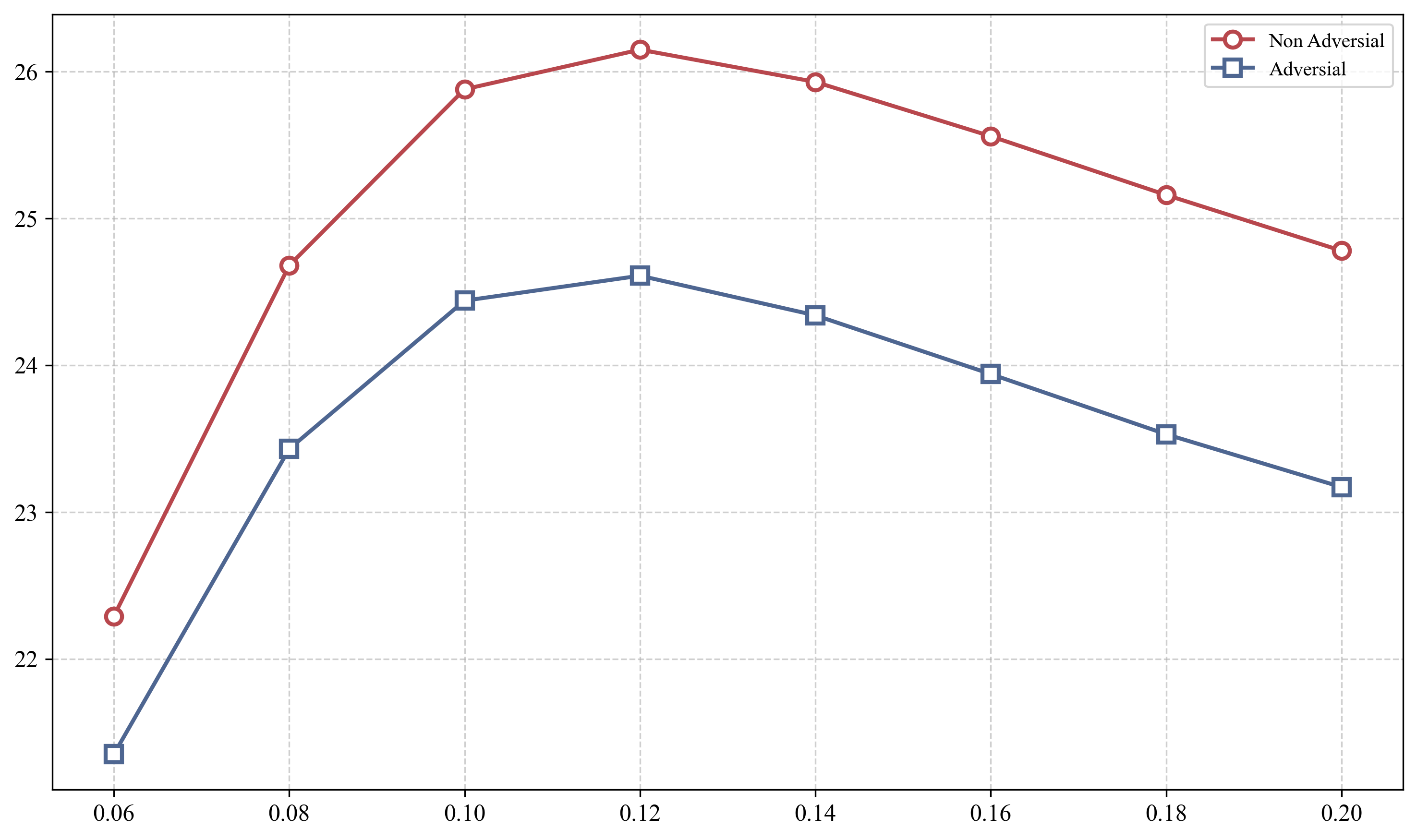} &
		\includegraphics[width=.48\textwidth]{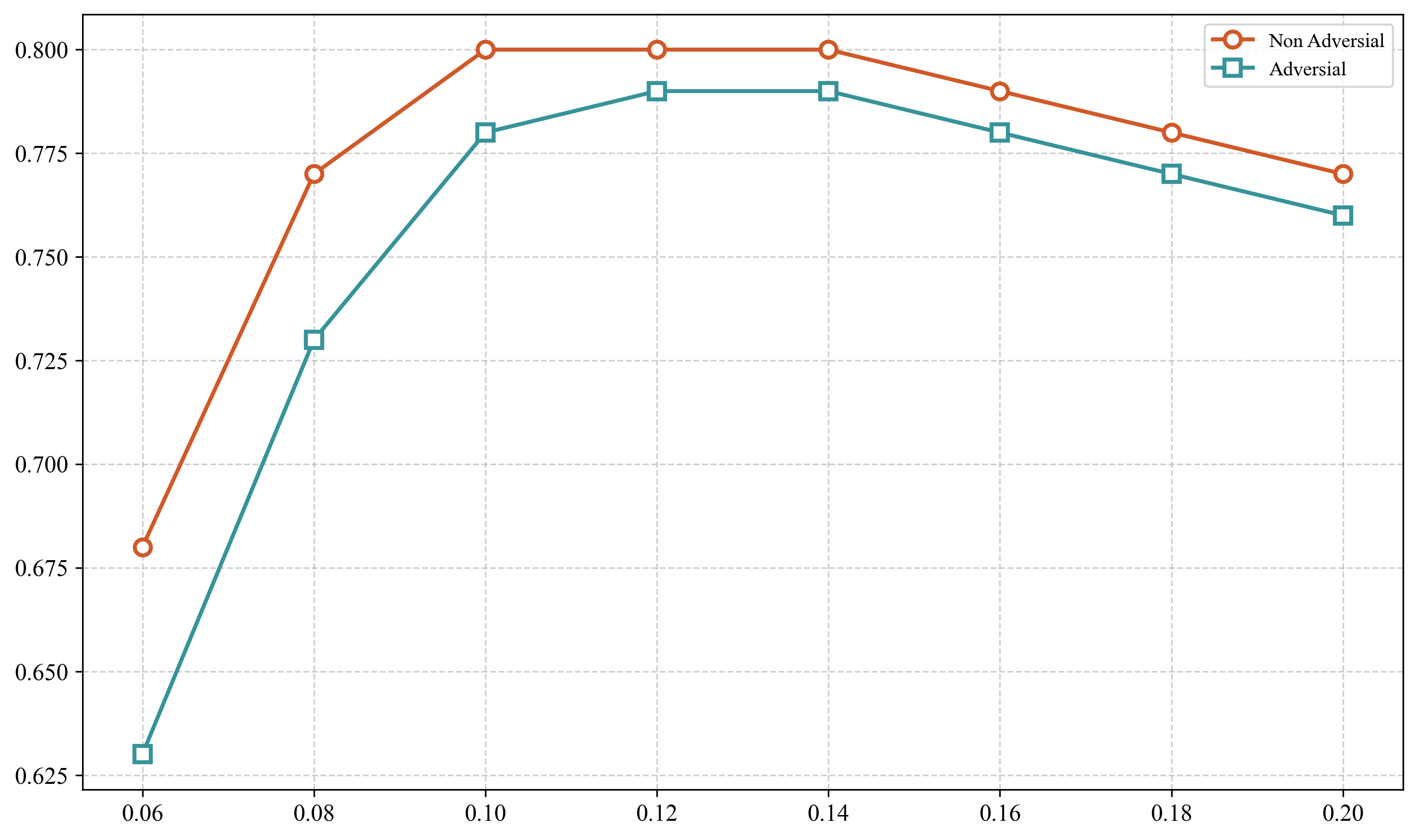} \\[-5pt]

        \footnotesize{\tiny PSNR under different $\tau$} & 
		\footnotesize{\tiny SSIM  under different $\tau$} \\

      \end{tabular}
      \caption{The values of PSNR and SSIM of images restored from adversarial {samples} and non adversarial images at different $\tau$.}
	\label{psnr_ssim_tau}
\end{figure*}

Considering the extreme situation $\tau=0$ (i.e., the diffusion equation is completely removed). An ablation study is required to investigate this situation. After retraining our network on the same training dataset with $\tau=0$ and $K=5$, we test its performance on the noisy images and  their corresponding adversarial {samples}, respectively. The results are shown in Figure \ref{ablation}.

\begin{figure*}[htbp]
        \centering
    \begin{tabular}{@{}c@{~}c@{~}c@{~}c@{~}c@{}}
		\includegraphics[width=.18\textwidth]{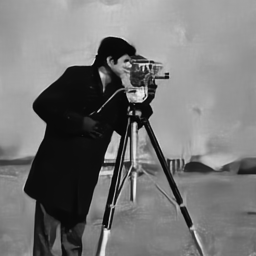} &
		\includegraphics[width=.18\textwidth]{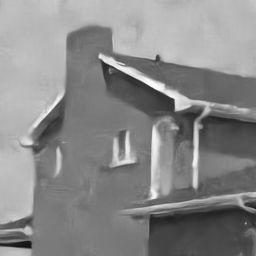} &
        \includegraphics[width=.18\textwidth]{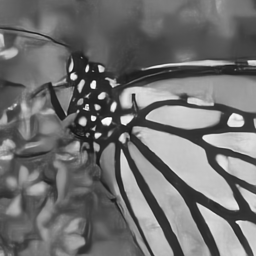} &
		\includegraphics[width=.18\textwidth]{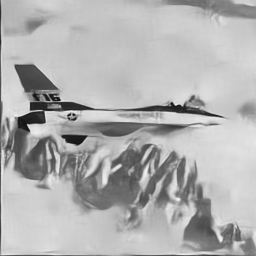} &
		\includegraphics[width=.18\textwidth]{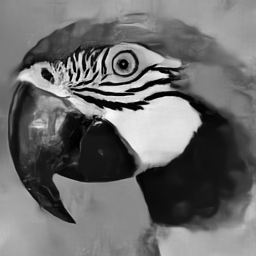} \\
        
        \footnotesize{\tiny PSNR=26.19} & 
		\footnotesize{\tiny PSNR=28.20} & 
		\footnotesize{\tiny PSNR=26.01}  & 
		\footnotesize{\tiny PSNR=23.96} & 
		\footnotesize{\tiny PSNR=25.98}\\

        \footnotesize{\tiny SSIM=0.80} & 
		\footnotesize{\tiny SSIM=0.81} & 
		\footnotesize{\tiny SSIM=0.85}  & 
		\footnotesize{\tiny SSIM=0.76} & 
		\footnotesize{\tiny SSIM=0.82}\\

		\includegraphics[width=.18\textwidth]{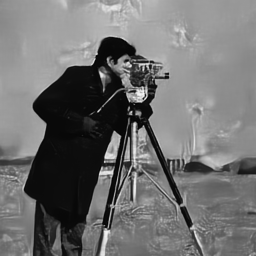} &
		\includegraphics[width=.18\textwidth]{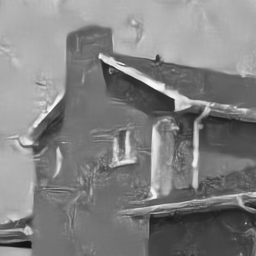} &
        \includegraphics[width=.18\textwidth]{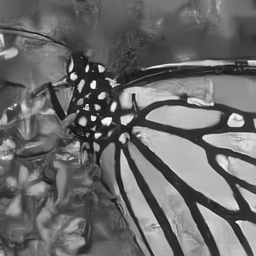} &
		\includegraphics[width=.18\textwidth]{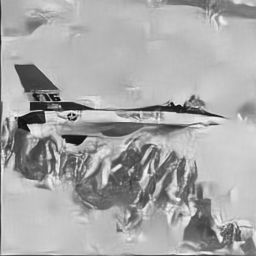} &
		\includegraphics[width=.18\textwidth]{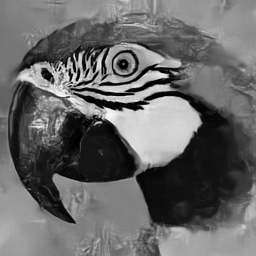} \\
        
        \footnotesize{\tiny PSNR=24.32} & 
		\footnotesize{\tiny PSNR=24.98} & 
		\footnotesize{\tiny PSNR=24.48}  & 
		\footnotesize{\tiny PSNR=22.12} & 
		\footnotesize{\tiny PSNR=24.05}\\

        \footnotesize{\tiny SSIM=0.78} & 
		\footnotesize{\tiny SSIM=0.76} & 
		\footnotesize{\tiny SSIM=0.81}  & 
		\footnotesize{\tiny SSIM=0.74} & 
		\footnotesize{\tiny SSIM=0.78}\\

      \end{tabular}
      \caption{Comparison of restoration results after our model ablation of $\tau$. The first line and second line show the results of the model on noisy images and adversarial {samples}, respectively.}
	\label{ablation}
\end{figure*}

Without diffusion blocks, the retrained model loses time-step adjustability.  Figure \ref{ablation} shows that the retrained network without  diffusion blocks can not resist adversarial attacks. 
Compared to the previous model ($\tau=0.12$), the retrained model achieves a 0.04 improvement in PSNR value with the same SSIM on the same noisy image `Cameraman' in Figure \ref{tau_impacet_no_adv} and  Figure \ref{ablation}.
However, with adjustable time steps, our model provides better edge preservation ($\tau=0.08$) visually.
For the same adversarial input `Cameraman' in  Figure \ref{tau_impacet} and Figure \ref{ablation}, the PSNR and SSIM of the retrained model are 0.29 and 0.01 lower than those of the original model ($\tau=0.12$).
Visually, the retrained model displays undesirable oscillations.	
In summary, diffusion blocks provide greater flexibility, enhanced adversarial robustness, and adjustable image quality.

\subsection{Results on real SAR images}
This section presents the performance comparison of all models on real SAR images.	
Figures \ref{real-sar-image1}--\ref{real-sar-image3} compare the restoration results of our method with other competing models.	
We employ three distinct time steps, $\tau = 0.020$, 0.100, and 0.180, to highlight the adaptability of our model in controlling result smoothness, as illustrated in the final three images of Figures \ref{real-sar-image1}--\ref{real-sar-image3} .

In Figures \ref{real-sar-image1}--\ref{real-sar-image3} , the results given by  AGSDNet and SAR-RDCP are not satisfactory, since 
the results from them retain significant noise.	
Moreover, while SAR-CNN is more effective at removing noise compared to the above two models, it introduces white spots that significantly affect both the visual quality and the calculation of ENL and Cx.	
Mulog-BM3D, TB-SAR, and our model yield acceptable results in terms of balancing noise removal and detail preservation.

\begin{table}[htbp]\tiny
    \centering
    \caption{The comparison of ENL, Cx and EPD-ROA on real SAR images under different atmospheric conditions}\label{real-sar-img-diff-ats}
\begin{tabular}{cccccc}
    \hline
    \multicolumn{1}{c}{\multirow{2}{*}{Images}}             & \multicolumn{1}{c}{\multirow{2}{*}{Methods}} & \multicolumn{1}{c}{\multirow{2}{*}{ENL}} & \multicolumn{1}{c}{\multirow{2}{*}{Cx}}& \multicolumn{2}{c}{EPD-ROA}       \\ \cline{5-6} 
    \multicolumn{1}{c}{}                                     & \multicolumn{1}{c}{}                       				& \multicolumn{1}{c}{}                       				& \multicolumn{1}{c}{}                       & \multicolumn{1}{c}{HD}   & \multicolumn{1}{c}{VD}\\ 
    \hline
    \multirow{3}{*}{Front1}
    &Ours-$\tau=0.100$&2.677&0.611&0.456&0.490\\
    &Ours-$\tau=0.180$&2.743&0.604&0.451&0.486\\
    \hline
    \multirow{3}{*}{Front2} 
    &Ours-$\tau=0.100$&1.870&0.731& 0.414&0.397 \\
    &Ours-$\tau=0.180$&1.847&0.736&0.405 & 0.386\\
    \hline
    \multirow{3}{*}{Rain} 
    &Ours-$\tau=0.100$&7.234&0.372&0.450 &0.424 \\
    &Ours-$\tau=0.180$&7.474&0.366&0.448 &0.421 \\
    \hline

\end{tabular}
\end{table}

Compared to other approaches, our model achieves an optimal trade-off between noise removal and detail preservation, owing to its ability to adjust smoothness adaptively.	
Smaller time steps prioritize detail preservation, whereas larger time steps enhance noise reduction, as evidenced by the final three images in Figures \ref{real-sar-image1}--\ref{real-sar-image3}. 	
Unlike traditional neural networks that the outputs are fixed, the flexibility of our model significantly enhances the controllability of the results.	
\begin{figure}[htbp]
    \centering
    \subfigure[\tiny Noisy]{
            \includegraphics[width=0.8in]{figures_real-sar_sar1_Noisy_ICEYE_1.png}
    }%
    \subfigure[\tiny MuLoG-BM3D]{
            \includegraphics[width=0.8in]{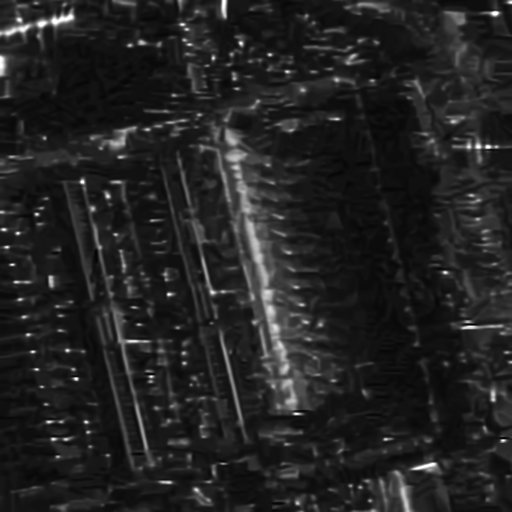}
    }%
    \subfigure[\tiny SAR-CNN]{
            \includegraphics[width=0.8in]{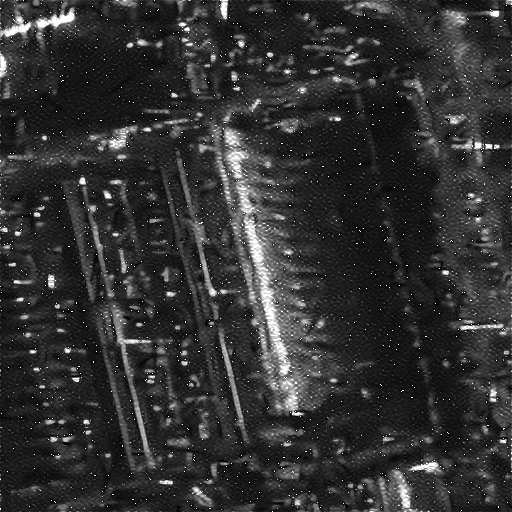}
    }%
    \subfigure[\tiny AGSDNet]{
            \includegraphics[width=0.8in]{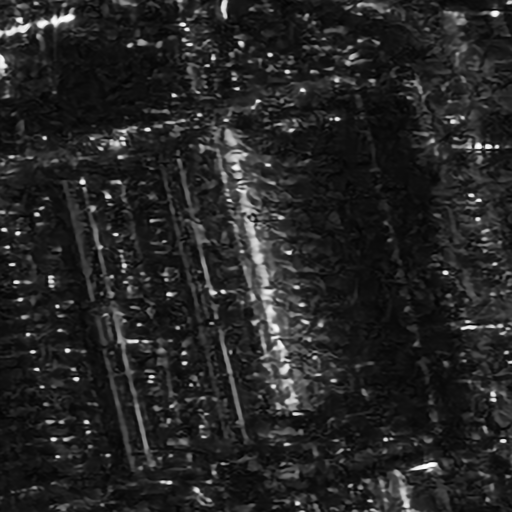}
    }%
    \subfigure[\tiny SAR-RDCP]{
            \includegraphics[width=0.8in]{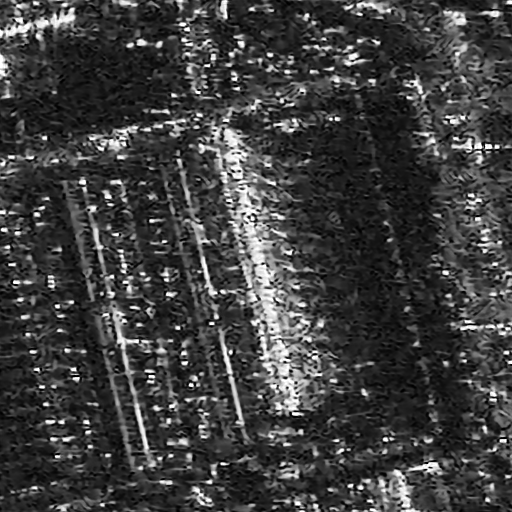}
    } \\
    \subfigure[\tiny TB-SAR]{
            \includegraphics[width=0.8in]{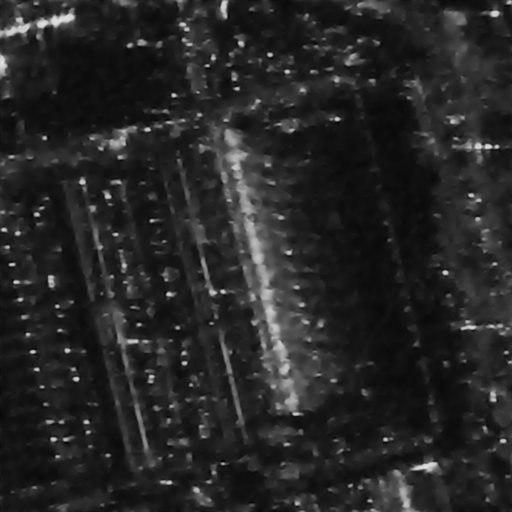}
    }%
    \subfigure[\tiny Ours$^{0.020}$]{
            \includegraphics[width=0.8in]{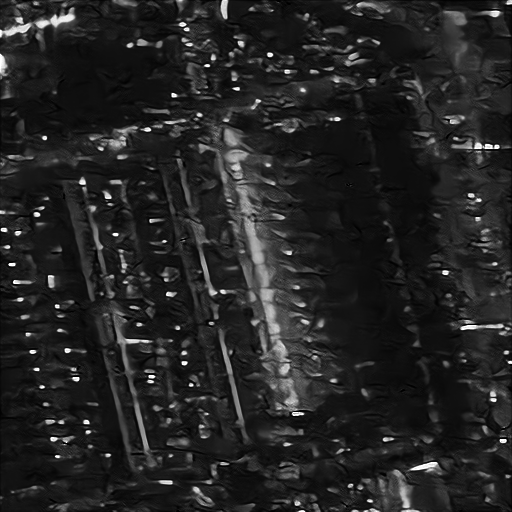}
    }%
    \subfigure[\tiny Ours$^{0.100}$]{
            \includegraphics[width=0.8in]{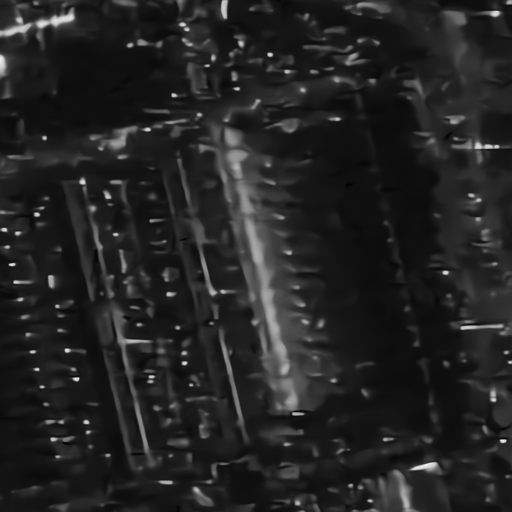}
    }%
    \subfigure[\tiny Ours$^{0.180}$]{
            \includegraphics[width=0.8in]{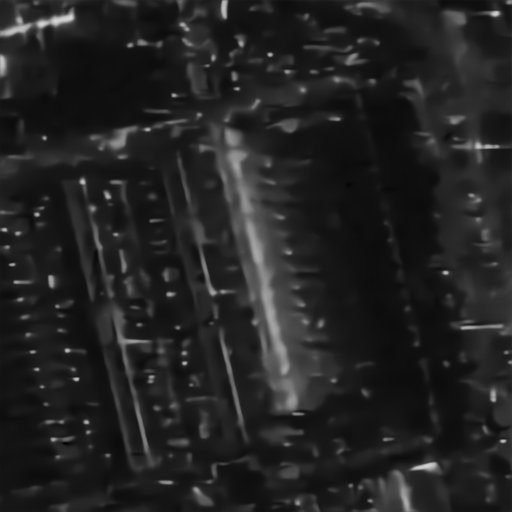}
    }%

    \centering
    \caption{Restoration results for SAR1.}
    \label{real-sar-image1}
\end{figure}
\begin{figure}[htbp]
    \centering
    \subfigure[\tiny Noisy]{
            \includegraphics[width=0.8in]
            {figures_real-sar_sar2_Noisy_ICEYE_2.png}
    }%
    \subfigure[\tiny MuLoG-BM3D]{
            \includegraphics[width=0.8in]{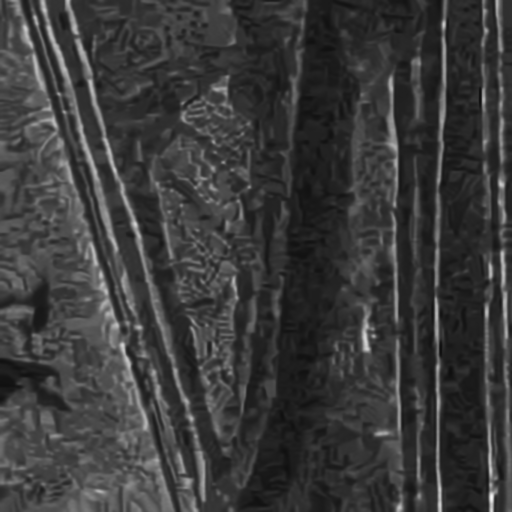}
    }%
    \subfigure[\tiny SAR-CNN]{
            \includegraphics[width=0.8in]{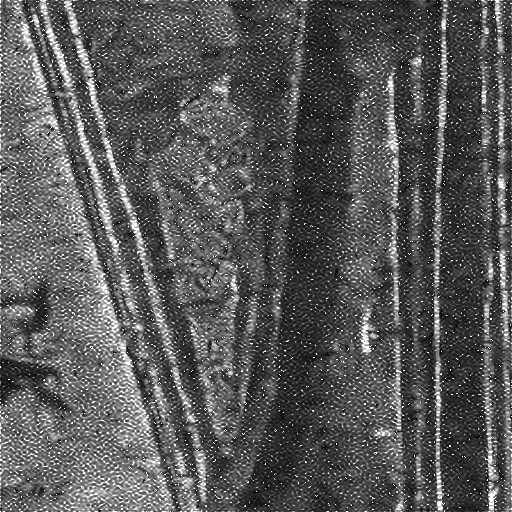}
    }%
    \subfigure[\tiny AGSDNet]{
            \includegraphics[width=0.8in]{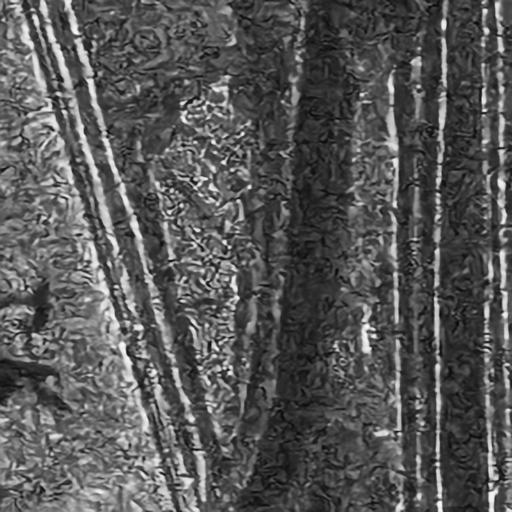}
    }%
    \subfigure[\tiny SAR-RDCP]{
            \includegraphics[width=0.8in]{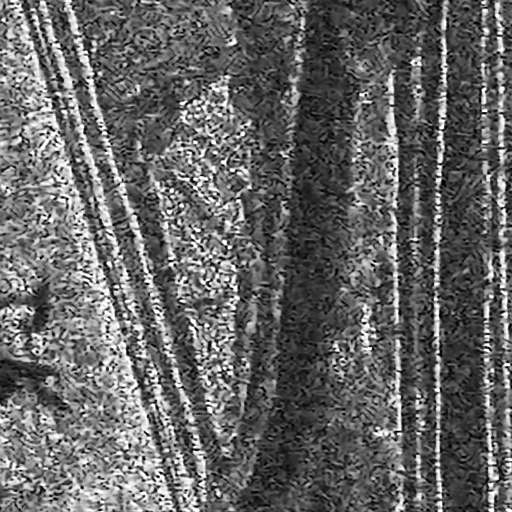}
    } \\
    \subfigure[\tiny TB-SAR]{
            \includegraphics[width=0.8in]{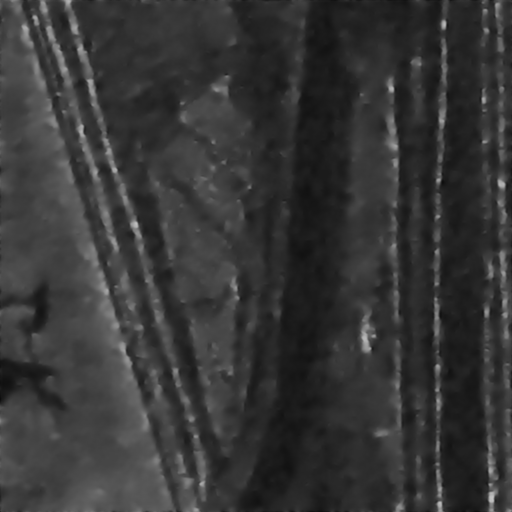}
    }%
    \subfigure[\tiny Ours$^{0.020}$]{
            \includegraphics[width=0.8in]{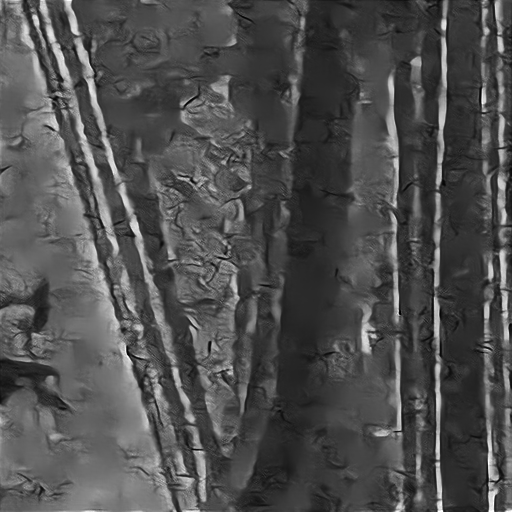}
    }%
    \subfigure[\tiny Ours$^{0.100}$]{
            \includegraphics[width=0.8in]{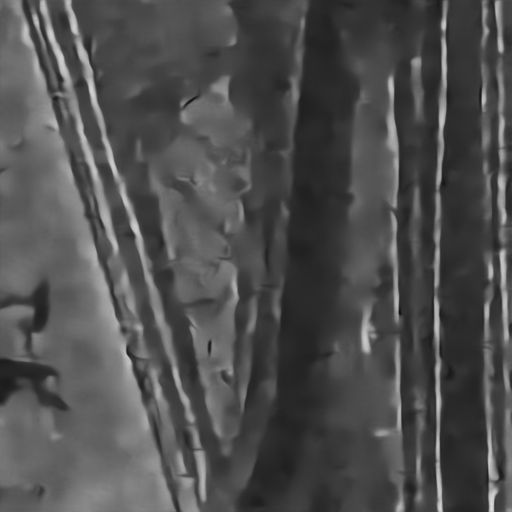}
    }%
    \subfigure[\tiny Ours$^{0.180}$]{
            \includegraphics[width=0.8in]{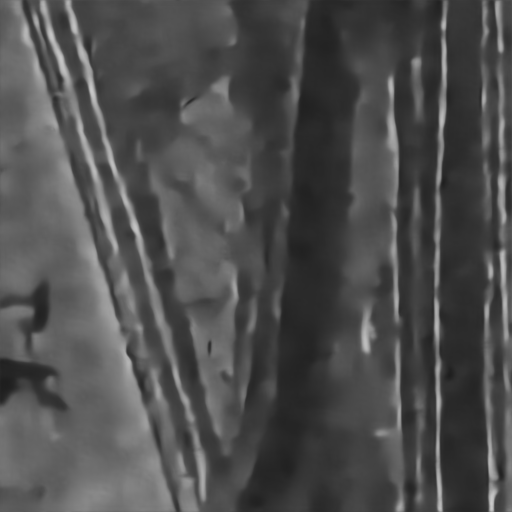}
    }%

    \centering
    \caption{Restoration results for SAR2.}
    \label{real-sar-image2}
\end{figure}

Among all methods, our model demonstrates superior noise removal at $\tau \geqslant  0.1$ and performs comparably to Mulog-BM3D in detail preservation at $\tau = 0.02$.	
For example, in the bottom-right region of SAR3, Mulog-BM3D restores a white square with black lines.	
Our model effectively reconstructs the primary lines in this region, whereas TB-SAR overly smooths finer details.	
While Mulog-BM3D excels at detail preservation, its outputs exhibit undesirable ripples, notably on the left side of SAR2.	
\begin{figure}[htbp]
    \centering
    \subfigure[\tiny Noisy]{
            \includegraphics[width=0.8in]
            {figures_real-sar_sar3_Noisy_NH49E008024.png}
    }%
    \subfigure[\tiny MuLoG-BM3D]{
            \includegraphics[width=0.8in]{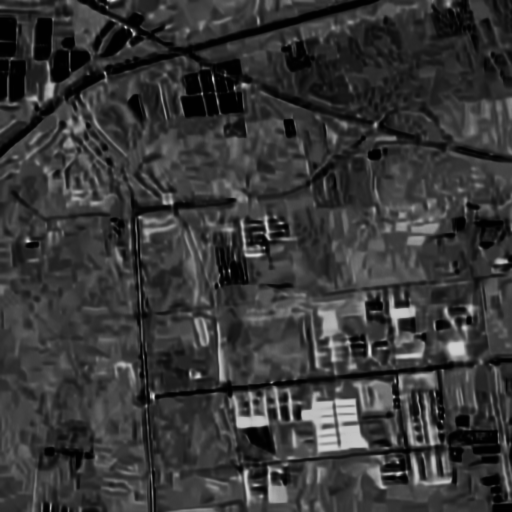}
    }%
    \subfigure[\tiny SAR-CNN]{
            \includegraphics[width=0.8in]{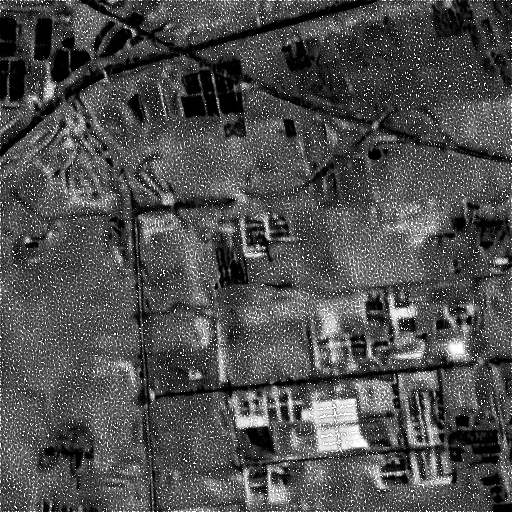}
    }%
    \subfigure[\tiny AGSDNet]{
            \includegraphics[width=0.8in]{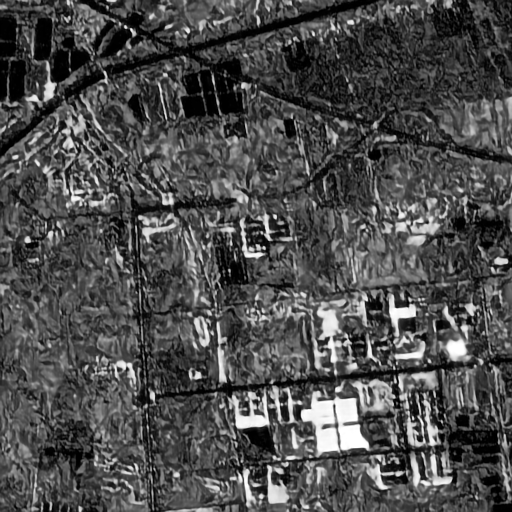}
    }%
    \subfigure[\tiny SAR-RDCP]{
            \includegraphics[width=0.8in]{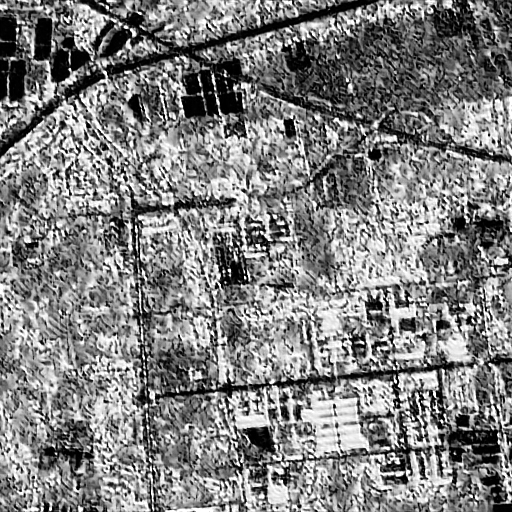}
    } \\
    \subfigure[\tiny TB-SAR]{
            \includegraphics[width=0.8in]{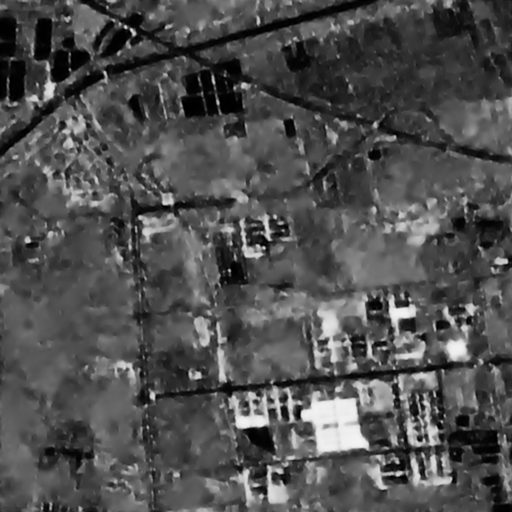}
    }%
    \subfigure[\tiny Ours$^{0.020}$]{
            \includegraphics[width=0.8in]{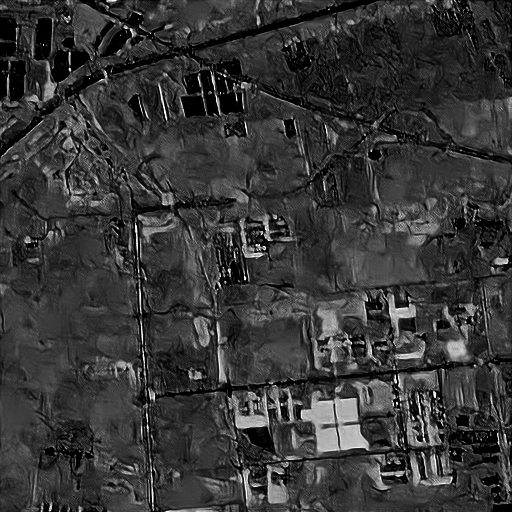}
    }%
    \subfigure[\tiny Ours$^{0.100}$]{
            \includegraphics[width=0.8in]{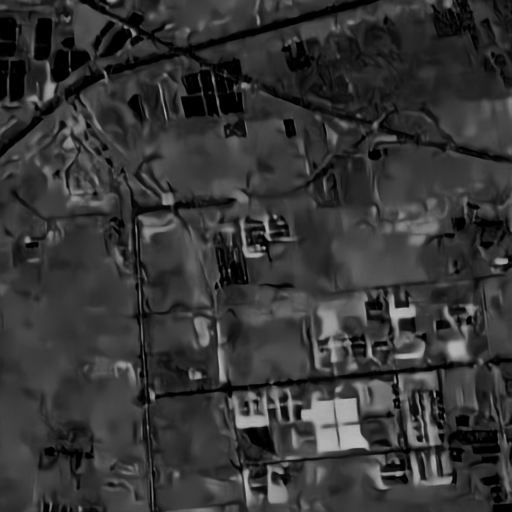}
    }%
    \subfigure[\tiny Ours$^{0.180}$]{
            \includegraphics[width=0.8in]{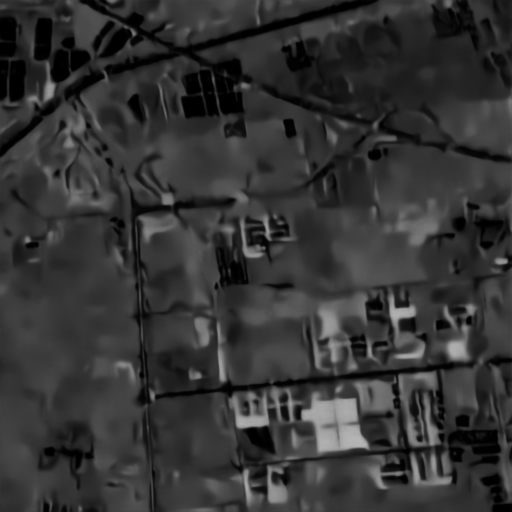}
    }%

    \centering
    \caption{Restoration results for SAR3.}
    \label{real-sar-image3}
\end{figure}

Table \ref{real-sar-img} summarizes the ENL, Cx and RPD-ROA metrics for the results of all methods on the three SAR images.	
In Table \ref{real-sar-img}, the best values are highlighted in bold.	
It is important to note that the images restored by SAR-CNN contain numerous unintended white specks, significantly distorting its performance metrics. As a result, the calculated ENL and Cx metrics for SAR-CNN are unreliable and provide limited informative value.	
Moreover, similar results were obtained by both TB-SAR and our method in terms of the denoising metrics ENL and Cx, which reflects the cleanliness of denoising, and also corresponds to the visual results in Figures \ref{real-sar-image1}--\ref{real-sar-image3}.	
The SAR-RDCP model achieves the highest EPD-ROA, indicating superior detail retention, followed by our model, which secures the second-highest rank.	
The accumulative regularization in our model ensures smooth outputs, which explains its slightly lower EPD-ROA performance compared to SAR-RDCP.

\begin{table}[htbp]\tiny
    \centering
    \caption{The comparison of ENL, Cx, and EPD-ROA on real SAR images}\label{real-sar-img}
\begin{tabular}{cccccc}
    \hline
    \multicolumn{1}{c}{\multirow{2}{*}{Images}}             & \multicolumn{1}{c}{\multirow{2}{*}{Methods}} & \multicolumn{1}{c}{\multirow{2}{*}{ENL}} & \multicolumn{1}{c}{\multirow{2}{*}{Cx}}& \multicolumn{2}{c}{EPD-ROA}       \\ \cline{5-6} 
    \multicolumn{1}{c}{}                                     & \multicolumn{1}{c}{}                       				& \multicolumn{1}{c}{}                       				& \multicolumn{1}{c}{}                       & \multicolumn{1}{c}{HD}   & \multicolumn{1}{c}{VD}\\ 
    \hline
    \multirow{8}{*}{SAR1} &Mulog-BM3D&1.618&0.786&0.571&0.579 \\
    &SAR-CNN&0.001&29.760&0.640&0.646\\
    &AGSDNet&1.282&0.883&0.595&0.621\\
    &SAR-RDCP&1.524&0.809&\textbf{0.688}&\textbf{0.782}\\
    &TB-SAR&1.859&0.733&0.569&0.574\\
    &Ours-$\tau=0.020$&1.250&0.894&0.602&0.664\\
    &Ours-$\tau=0.100$&1.873&0.731&0.568&0.575\\
    &Ours-$\tau=0.180$&\textbf{2.281}&\textbf{0.662}&0.564&0.565\\
    \hline
    \multirow{8}{*}{SAR2} &Mulog-BM3D&4.545&0.469&0.604&0.594\\
    &SAR-CNN&0.001&63.351&0.720& 0.729 \\
    &AGSDNet&3.795&0.513&0.624&0.620 \\
    &SAR-RDCP&3.112&0.567&\textbf{0.740}&\textbf{0.751}\\
    &TB-SAR&6.436&0.394&0.595&0.591\\
    &Ours-$\tau=0.020$&4.805&0.456&0.618&0.602\\
    &Ours-$\tau=0.100$&6.467&0.393&0.598&0.590\\
    &Ours-$\tau=0.180$&\textbf{7.277}&\textbf{0.371}&0.595&0.589\\
    \hline
    \multirow{8}{*}{SAR3} &Mulog-BM3D&2.889&0.588&0.408&0.400\\
    &SAR-CNN&0.143&2.640&0.539&0.532 \\
    &AGSDNet&2.267&0.664&0.529& 0.481\\
    &SAR-RDCP&2.844&0.593&\textbf{0.899}& 0.697\\
    &TB-SAR&3.665&0.522&0.400& 0.422 \\
    &Ours-$\tau=0.020$&2.296&0.660&0.845& \textbf{0.720}\\
    &Ours-$\tau=0.100$&3.416&0.541&0.446&0.432\\
    &Ours-$\tau=0.180$&\textbf{3.899}&\textbf{0.506}&0.395&0.393\\
    \hline

\end{tabular}
\end{table}

\begin{figure}[htbp]
    \centering
    \subfigure[Front1]{
            \includegraphics[width=1in]{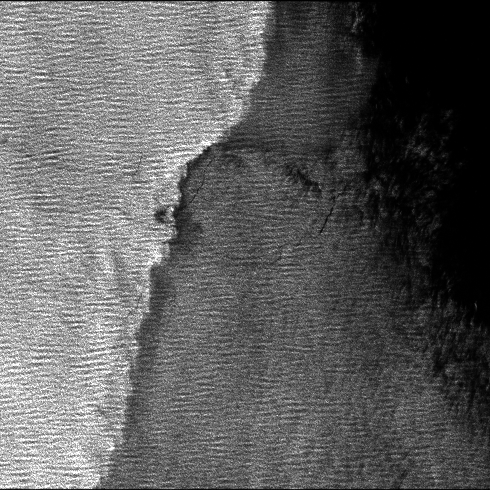}
            \label{Front1}
    }%
    \subfigure[Front2]{
            \includegraphics[width=1in]{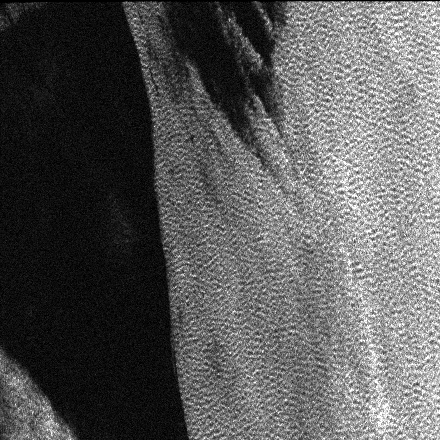}
            \label{Front2}
    }%
    \subfigure[Rain]{
            \includegraphics[width=1in]{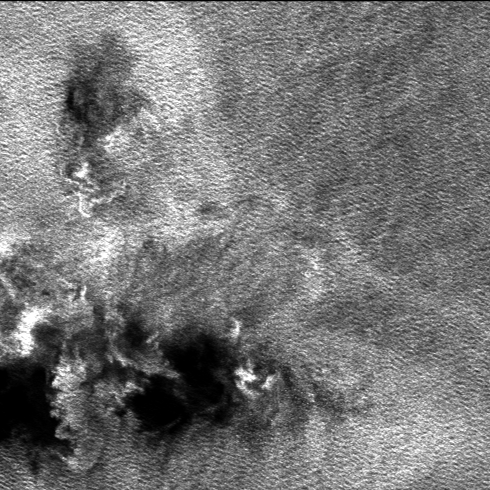}
            \label{Rain}
    }%
    \centering
    \caption{Real SAR images under different atmospheric conditions.}
    \label{real-sar-diff-ats}
\end{figure}

\begin{figure*}[htbp]
    \centering
    \begin{tabular}{@{}c@{~}c@{~}c@{}}

        \includegraphics[width=.2\textwidth]{figures_different_ats_ats_1_noisy.png} &
        \includegraphics[width=.2\textwidth]{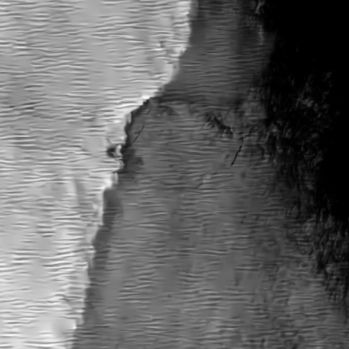} &
        \includegraphics[width=.2\textwidth]{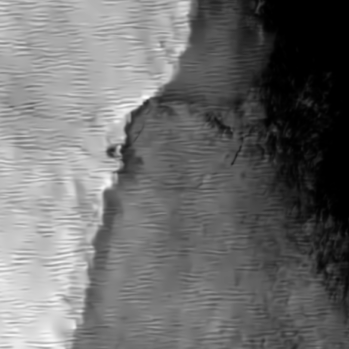} \\
        
        \includegraphics[width=.2\textwidth]{figures_different_ats_ats_2_atmos-3.png} &
        \includegraphics[width=.2\textwidth]{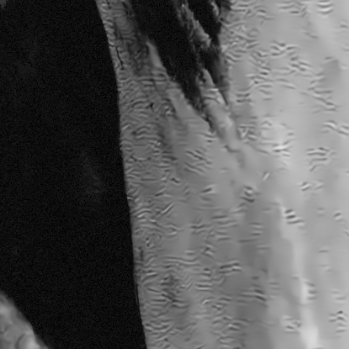} &
        \includegraphics[width=.2\textwidth]{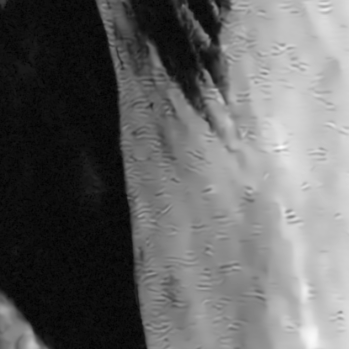} \\
        
        \includegraphics[width=.2\textwidth]{figures_different_ats_rain_rain.png} &
        \includegraphics[width=.2\textwidth]{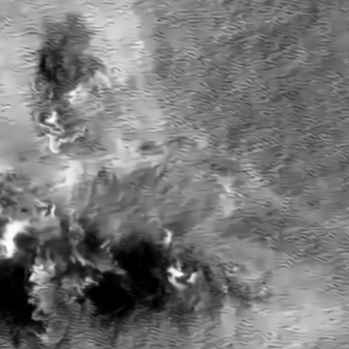} &
        \includegraphics[width=.2\textwidth]{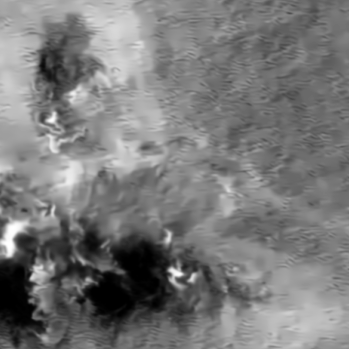} \\

        \footnotesize{\tiny Nosiy} & 
        \footnotesize{\tiny Ours$^{0.10}$}  & 
        \footnotesize{\tiny Ours$^{0.18}$} \\

      \end{tabular}
      \caption{Restoration results under different atmospheric conditions. }
      \label{rain_astom}
\end{figure*}

    To evaluate the restoration results of our model under real-world disturbances, we select three real SAR sea surface images captured under different atmospheric conditions as test images \footnote{ https://www.seanoe.org/data/00456/56796 }.
    These images are shown in Figure \ref{real-sar-diff-ats}.	
    One of these images is captured during rainfall, while the others are  obtained during atmospheric front conditions.	
    We find that the noise types in Figure \ref{real-sar-diff-ats} differ from  the real SAR images shown in Figure \ref{real-sar}. 
    The noise patterns under atmospheric front and rainfall conditions tend to be more strip-shaped.	
    This indicates that the noise distribution in Figure \ref{real-sar-diff-ats} differs from that in Figure \ref{real-sar}.

    The restoration results at different time steps are presented in Figure \ref{rain_astom}.	
    As shown in Figure \ref{rain_astom}, the restoration results are generally acceptable, and adjusting the step size effectively controls output smoothness.	
    Our model performs well {on} `Front2' and `Rain' images at $\tau=0.18$, as it effectively removes noise while preserving details.	
    For `Front1', although details are preserved and {the} image quality is improved, {the} noise removal is incomplete, {leaving} some strip-shaped noise.	
    Table \ref{real-sar-img-diff-ats} summarizes the ENL, Cx, and EPD-ROA values.	Our model achieves the higher EPD-ROA value at $\tau=0.10$. 
    The ENL (Cx)  value gets higher (lower) when $\tau=0.18$ on `Front1' and `Rain'.

    Considering both visual quality and evaluation metrics, our model effectively removes noise from real SAR images while maintaining excellent detail preservation.	
    Most importantly, our model exhibits high flexibility compared to other traditional neural networks. 

\subsection{Differences between our strategy and adversarial training}

As illustrated in contribution 1, our model is trained exclusively on a training dataset contaminated by multiplicative Gamma noise, without any adversarial samples in the training process. 
However, the core principle of adversarial training is to add adversarial samples to the train set during the training process. 
This strategy forces the model to develop resistance against input disturbances and improves the robustness of the model. 
In this process, compared to our method, both generating and training with adversarial samples require more time and memory. 
\begin{table}[htbp] \tiny
    \centering
	\caption{Comparison of our training strategy and adversarial training strategy}\label{adv_training_compare}
	\begin{tabular}{ccccccc}   
	\toprule
	\text{Training strategy} & \makecell[c]{Time \\ consumption} & \text{Memory} & \makecell[c]{PSNR/SSIM \\ (no attack)} & \makecell[c]{PSNR/SSIM  \\($L^\infty$-PGD)}& \makecell[c]{PSNR/SSIM \\(MIFGSM)}\\ 
	\midrule
	
    Adversarial training	&898.52s	&6250.00 MB	&	25.92/0.77	&	25.23/0.76&23.23/0.76\\ 

    Our strategy &407.55s & 3125.00MB&25.78/0.77 &24.19/0.77& 24.04/0.75\\

	 \bottomrule
	\end{tabular}
	\end{table}

    \begin{figure*}[htbp]
        \centering
        \begin{tabular}{@{}c@{~}c@{~}c@{}}

            \includegraphics[width=.2\textwidth]{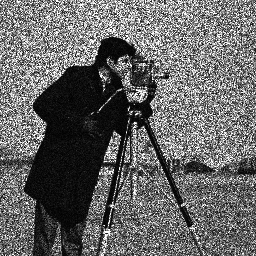} &
            \includegraphics[width=.2\textwidth]{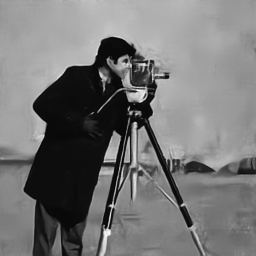} &
            \includegraphics[width=.2\textwidth]{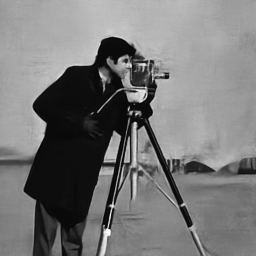} \\

            \includegraphics[width=.2\textwidth]{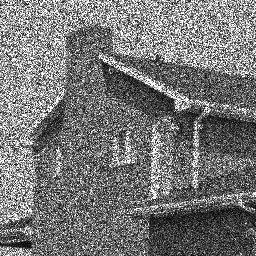} &
            \includegraphics[width=.2\textwidth]{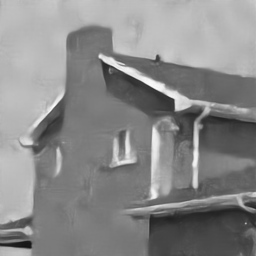} &
            \includegraphics[width=.2\textwidth]{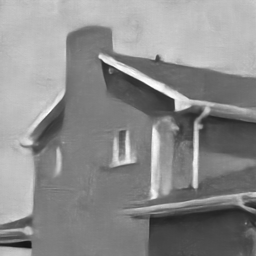} \\
    
            \footnotesize{\tiny {Nosiy}} & 
            \footnotesize{\tiny {Ours}} & 
            \footnotesize{\tiny {SAR-CNN}}  \\

          \end{tabular}
          \caption{{Restoration results on no adversarial attack images under our training strategy and adversarial training strategy.}}
          \label{adv_training_compare_vis}
    \end{figure*}

To better illustrate this difference, we conduct a comparative experiment with adversarial training. We retrain a 17 layers SAR-CNN under the adversarial training strategy. Specifically, a pretrained 10 layers SAR-CNN is used to generate adversarial samples using $L^\infty$-denoising PGD attack strategy during training, and these samples are incorporated into the training set to improve the robustness of SARCNN. 
The computational time and memory usage during the first epoch are presented in Table \ref{adv_training_compare}. 
As the training set size doubled, adversarial training requires nearly twice the computational resources.

    \begin{figure*}[htbp]
        \centering
        \begin{tabular}{@{}c@{~}c@{~}c@{~}c@{~}c@{~}c@{}}

            \includegraphics[width=.15\textwidth]{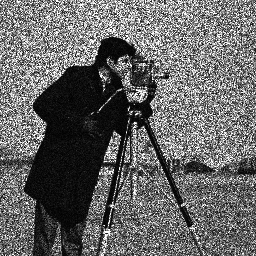} &
            \includegraphics[width=.15\textwidth]{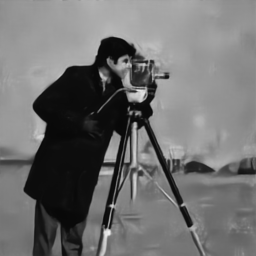} &
            \includegraphics[width=.15\textwidth]{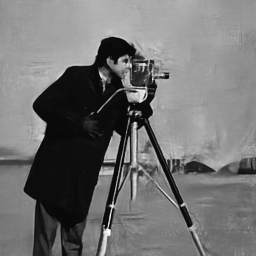}&
            \includegraphics[width=.15\textwidth]{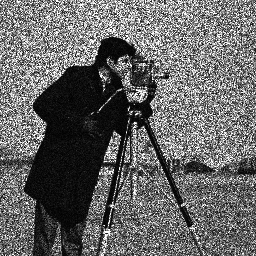} &
            \includegraphics[width=.15\textwidth]{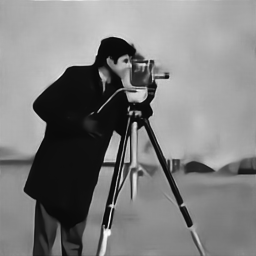} &
            \includegraphics[width=.15\textwidth]{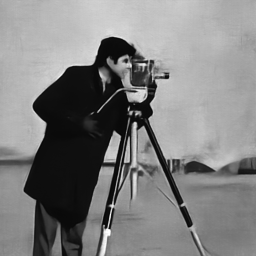} \\

            \includegraphics[width=.15\textwidth]{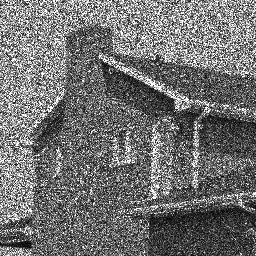} &
            \includegraphics[width=.15\textwidth]{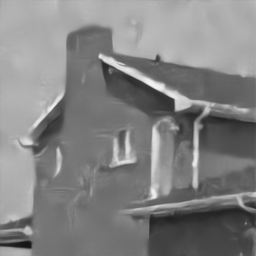} &
            \includegraphics[width=.15\textwidth]{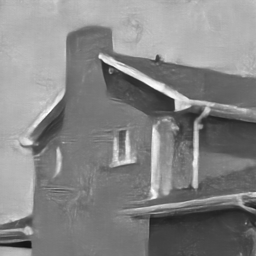} &
            \includegraphics[width=.15\textwidth]{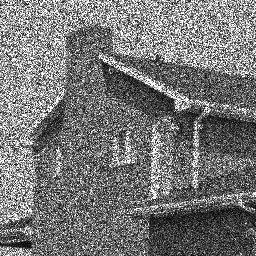} &
            \includegraphics[width=.15\textwidth]{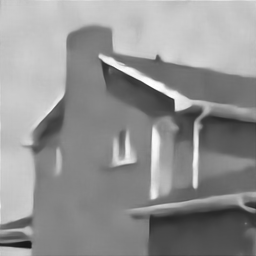} &
            \includegraphics[width=.15\textwidth]{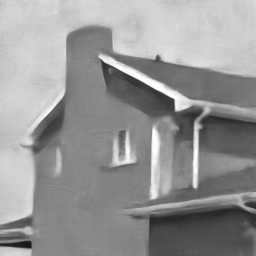}\\
    
            \raisebox{.5ex}{\makecell[c]{\footnotesize{\tiny Adversarial {sample}} \\[-5pt]\footnotesize{\tiny ($L^\infty$-PGD)}}}& 
            \footnotesize{\tiny Ours} & 
            \footnotesize{\tiny SAR-CNN}  & 
            \raisebox{.5ex}{\makecell[c]{\footnotesize{\tiny Adversarial {sample}} \\[-5pt]\footnotesize{\tiny (MIFGSM)}}}&
            \footnotesize{\tiny Ours}  &
            \footnotesize{\tiny SAR-CNN} \\
    
          \end{tabular}
          \caption{Restoration results on adversarial {samples} under our training strategy and adversarial training strategy.}
          \label{adv_training_compare_vis1}
    \end{figure*}
    
    After training, we evaluate both SAR-CNN (with adversarial training) and our model (without adversarial training) on the Set12 dataset.	
    To evaluate the generalization capability of these models, we test them on adversarial samples generated by both $L^\infty$-denoising PGD attacks and MIFGSM attacks \cite{dong2018boosting}. 
    The visualization results for noisy {images} and adversarial {samples} are presented in Figures \ref{adv_training_compare_vis} and \ref{adv_training_compare_vis1}, respectively.	
    Visually, both models demonstrate excellent performance on both non-adversarial {images} and adversarial {samples}.	
    Quantitative results in Table \ref{adv_training_compare}, which summarizes PSNR and SSIM values for different image types reveal significant differences between the models.	
    Although the SSIM values are comparable between the two methods, adversarially-trained SAR-CNN achieves higher PSNR {on} adversarial {samples} generated by the $L^\infty$-PGD strategy. 
    The network was exposed to adversarial {samples} generated by $L^\infty$-PGD strategy during training, allowing it to learn their noise patterns.
    However, as shown in the last column of Table \ref{adv_training_compare}, adversarially-trained SAR-CNN exhibits lower PSNR than our model for MIFGSM-generated images.	 
    In contrast, our model (without adversarial training) maintains consistent PSNR and SSIM values across various adversarial {samples}.	
    In summary, while adversarial training only provides resistance against the same type of attacks, our model maintains consistent performance against various attack types.	

\subsection{Decoupling}
In this section, to explore the deeper application of the denoising network block and diffusion block in our model, {we decouple these two blocks and check their functions  separately. 
Theoretically, our strategy is effective for any denoising network structure. In other words, the diffusion block does not depend on network type, and it can be used independently to increase the stability of the entire network, no matter what kind of denoising network it is. To verify this point, we integrate the diffusion block into a transformer-based network, and test its performance on denoising and stability.}

In fact, the denoising network block and diffusion block can be naturally decoupled because they are executed sequentially. The only connection is that the output of one block serves as the input to the other. The restoration results without diffusion equation blocks and denoising network blocks are shown in the Figure \ref{decoup_diff}.

\begin{figure*}[htbp]
        \centering
    \begin{tabular}{@{}c@{~}c@{~}c@{~}c@{~}c@{~}c@{}}
        \includegraphics[width=.16\textwidth]{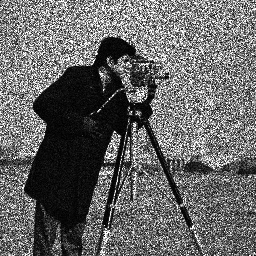} &
        \includegraphics[width=.16\textwidth]{figures_Ablation_no_adv_synt01_02.png} &
        \includegraphics[width=.16\textwidth]{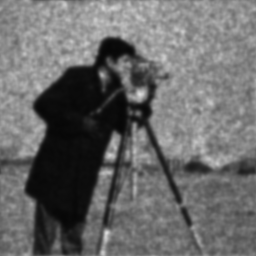} &
        \includegraphics[width=.16\textwidth]{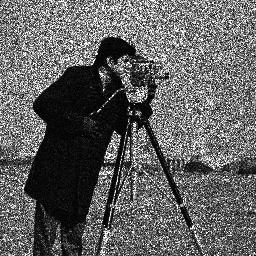} &
        \includegraphics[width=.16\textwidth]{figures_Ablation_adv_ADV_synt01_02.png} &
        \includegraphics[width=.16\textwidth]{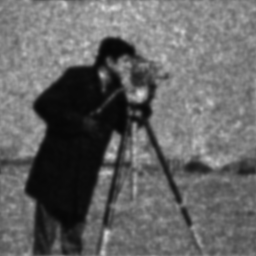}\\ 

        \includegraphics[width=.16\textwidth]{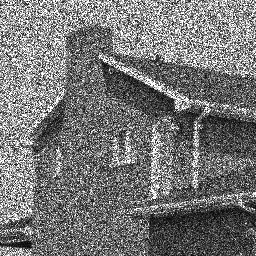} &
        \includegraphics[width=.16\textwidth]{figures_Ablation_no_adv_synt02_02.png} &
        \includegraphics[width=.16\textwidth]{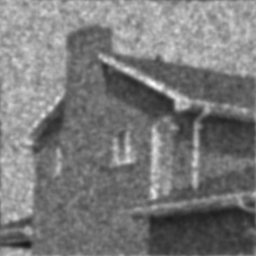} &
        \includegraphics[width=.16\textwidth]{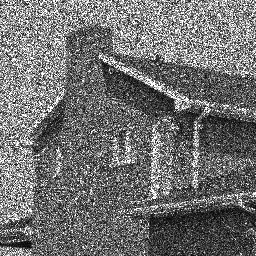} &
        \includegraphics[width=.16\textwidth]{figures_Ablation_adv_ADV_synt02_02.png} &
        \includegraphics[width=.16\textwidth]{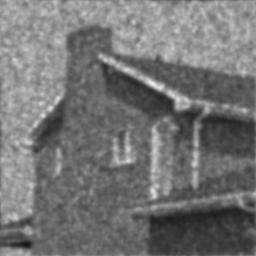}\\

        \footnotesize{\tiny Noisy {image}} & 
		\footnotesize{\tiny Only networks} & 
        \footnotesize{\tiny Only heat equation} & 
		\footnotesize{\tiny Adversarial {sample}} & 
		\footnotesize{\tiny Only networks} &
        \footnotesize{\tiny Only heat equation}  \\

      \end{tabular}
      \caption{The results after removing diffusion blocks and denoising network blocks. }
	\label{decoup_diff}
\end{figure*}

First, we examine the effect of the CNN block.	The ablation study in subsection \ref{sec:tau_impact} shows that when the diffusion block is removed and the network is retrained, it can still effectively restore noisy images.	
However, the network loses its adversarial attack resistance, as demonstrated in the second and fifth columns of Figure \ref{decoup_diff}.

The diffusion block can also function independently.	
In this case, the denoising algorithm degenerates to a classical PDE-based approach, producing results significantly worse than those of neural networks.	
The results of processing both noisy {images} and adversarial {samples} using only the heat equation are presented in the third and sixth columns of Figure \ref{decoup_diff}.	
This demonstrates that the heat equation cannot preserve edges during denoising and only smoothes the image. It is because of the linear and isotropic nature of the heat equation.

However, the diffusion block provides sufficient regularity to help neural networks resist adversarial attacks.	
More importantly, this block can be embedded into various neural networks.	
Embedding it into the neural networks several times  enhances {its} robustness.	
To verify effectiveness, we integrate the heat equation into a transformer-based multiplicative denoising network.	
Specifically, we select TB-SAR (mentioned earlier) as the target architecture and insert {the} heat {equation} after each transformer block, as shown in Figure \ref{Tb-sar-str}.	The specific structure of the `Overlap Patch Embedding', `Transformer Block', and `Conv Projection' inside the network can be found in \cite{perera2022transformer}. 
For simplicity, we denote this new network as TB-SAR-D.	
\begin{figure*}[htbp] 
    \centering  
    \includegraphics[width=1\textwidth]{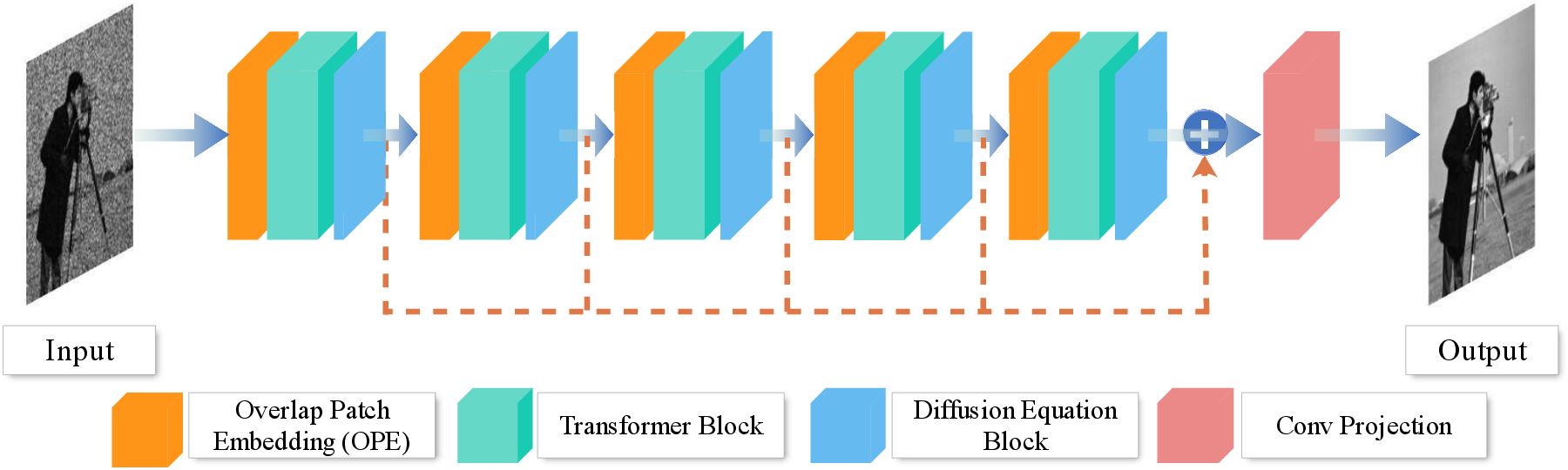}
    \caption{The structure of TB-SAR after inserting the diffusion blocks.}
    \label{Tb-sar-str}
\end{figure*}

We train TB-SAR and TB-SAR-D on the same training set with noise level $L=4$. 
Then, chossing test images from Set12 as inputs for both networks. 
Noisy {images} and adversarial {samples} are fed separately into these networks, with restoration results shown in Figure \ref{te_sar_D_L=4}.	
In Figure \ref{te_sar_D_L=4}, the second and third columns show the restoration results of TB-SAR and our model, respectively.	
The fourth to sixth columns correspond to the results of TB-SAR-D under preset $\tau$, $\tau=0.1$ and $\tau=0.3$, respectively.
Visually, TB-SAR shows incomplete denoising for adversarial samples, while increasing $\tau$ in TB-SAR-D mitigates this issue.

Integrating the heat equation into the TB-SAR neural network improves restoration quality and increases robustness.	
As shown in the second and fourth columns of Figure \ref{te_sar_D_L=4}, TB-SAR-D outperforms TB-SAR on noisy images.	
Moreover, as $\tau$ increases in TB-SAR-D, the PSNR value first rises and then declines, while the restored images become smoother. 
These results demonstrate that incorporating diffusion equations into the network during training improves restored image quality.	

Furthermore, the comparison between our model and TB-SAR-D illustrates that transformer-based models achieve higher PSNR values in restored results.	
CNN-based restoration yields better visual results than transformer-based networks.	
For adversarial sample restoration, CNN-based models are more effective at removing background noise and produce better visual results.

In summary, 
{the robustness improvement provided by the diffusion block is independent of the architecture of denoising network. Furthermore,}
CNN-based and transformer-based networks each have advantages {for} image restoration.	Most importantly, incorporating diffusion equations into {{the}} network training {{process}} enhances both network robustness and {{the quality of the restored images}}.	
\begin{figure*}[htbp]
    \centering
    \begin{tabular}{@{}c@{~}c@{~}c@{~}c@{~}c@{~}c@{~}c@{}}

        \raisebox{5ex}{\makecell[c]{{\footnotesize{a*}}}}&
        \includegraphics[width=.15\textwidth]{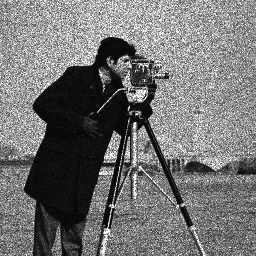} &
		\includegraphics[width=.15\textwidth]{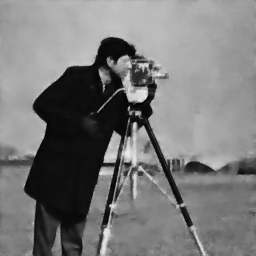} &
        \includegraphics[width=.15\textwidth]{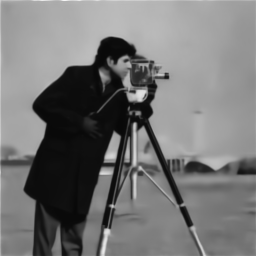} &
        \includegraphics[width=.15\textwidth]{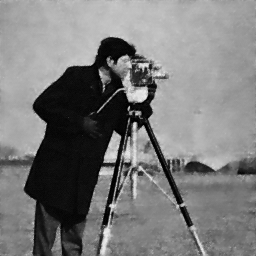} &
        \includegraphics[width=.15\textwidth]{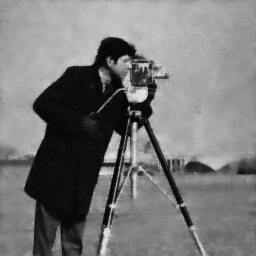} &
		\includegraphics[width=.15\textwidth]{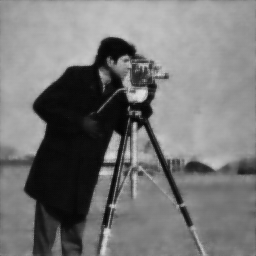} \\

        \footnotesize{ } &
		\footnotesize{\tiny PSNR=17.64} & 
		\footnotesize{\tiny PSNR=25.18} & 
        \footnotesize{\tiny PSNR=25.42} & 
		\footnotesize{\tiny PSNR=25.37}  & 
		\footnotesize{\tiny PSNR=26.12}  &
		\footnotesize{\tiny PSNR=25.14}  \\

        \raisebox{5ex}{\makecell[c]{{\footnotesize{b*}}}}&
        \includegraphics[width=.15\textwidth]{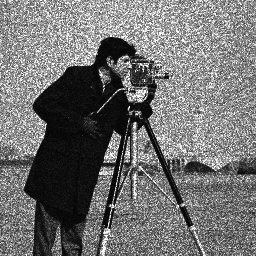} &
		\includegraphics[width=.15\textwidth]{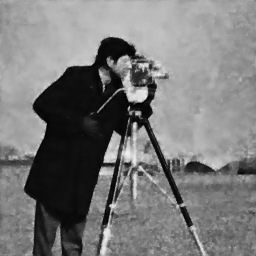} &
        \includegraphics[width=.15\textwidth]{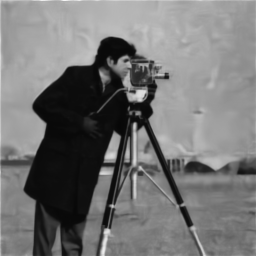} &
        \includegraphics[width=.15\textwidth]{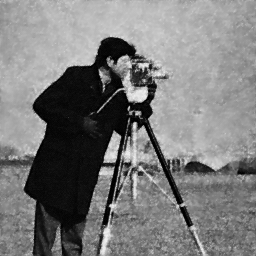} &
        \includegraphics[width=.15\textwidth]{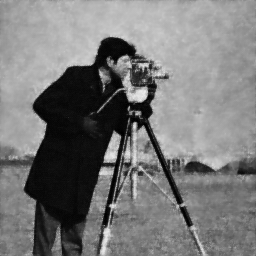} &
		\includegraphics[width=.15\textwidth]{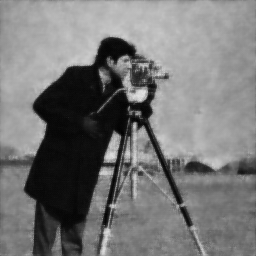} \\

        \footnotesize{ } &
		\footnotesize{\tiny PSNR=16.26} & 
		\footnotesize{\tiny PSNR=24.66} & 
        \footnotesize{\tiny PSNR=25.35} & 
		\footnotesize{\tiny PSNR=24.24}  & 
		\footnotesize{\tiny PSNR=25.67}  &
		\footnotesize{\tiny PSNR=25.14}  \\

        \raisebox{5ex}{\makecell[c]{{\footnotesize{a*}}}}&
        \includegraphics[width=.15\textwidth]{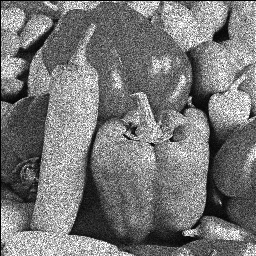} &
		\includegraphics[width=.15\textwidth]{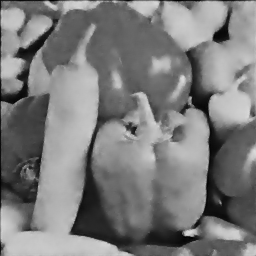} &
        \includegraphics[width=.15\textwidth]{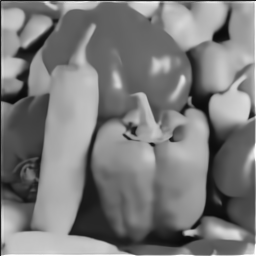} &
        \includegraphics[width=.15\textwidth]{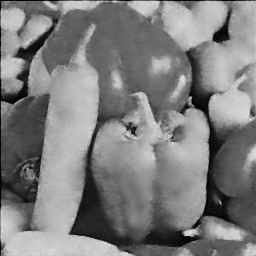} &
        \includegraphics[width=.15\textwidth]{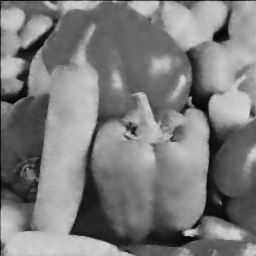} &
		\includegraphics[width=.15\textwidth]{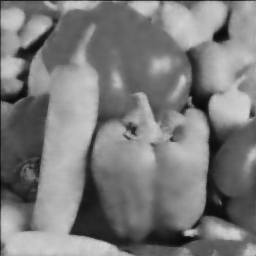} \\

        \footnotesize{ } &
		\footnotesize{\tiny PSNR=17.71} & 
		\footnotesize{\tiny PSNR=26.26} &
        \footnotesize{\tiny PSNR=24.50} &
		\footnotesize{\tiny PSNR=26.31}  & 
		\footnotesize{\tiny PSNR=26.58}  &
		\footnotesize{\tiny PSNR=25.54}  \\

        \raisebox{5ex}{\makecell[c]{{\footnotesize{b*}}}}&
        \includegraphics[width=.15\textwidth]{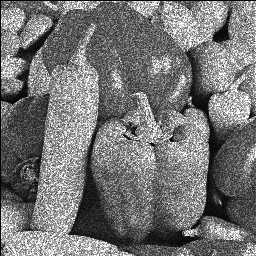} &
		\includegraphics[width=.15\textwidth]{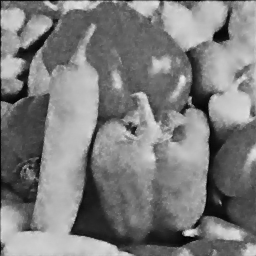} &
        \includegraphics[width=.15\textwidth]{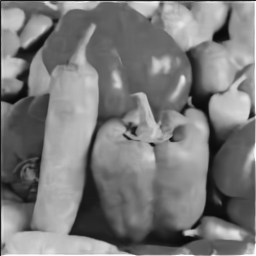} &
        \includegraphics[width=.15\textwidth]{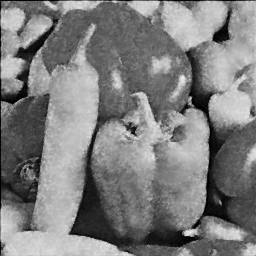} &
        \includegraphics[width=.15\textwidth]{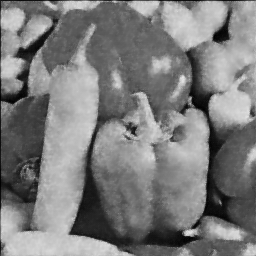} &
		\includegraphics[width=.15\textwidth]{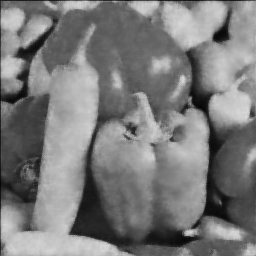} \\

        \footnotesize{ } &
		\footnotesize{\tiny PSNR=16.28} & 
		\footnotesize{\tiny PSNR=25.73} &
        \footnotesize{\tiny PSNR=24.45} & 
		\footnotesize{\tiny PSNR=25.14}  & 
		\footnotesize{\tiny PSNR=26.35}  &
		\footnotesize{\tiny PSNR=25.91}  \\

        \footnotesize{ } &
		\footnotesize{\tiny Nosiy} & 
		\footnotesize{\tiny TB-SAR} & 
        \raisebox{.3ex}{\makecell[c]{\footnotesize{\tiny Ours} \\[-5pt]\footnotesize{\tiny $\tau=0.23$}}}& 
		\footnotesize{\tiny TB-SAR-D}  & 
        \raisebox{.3ex}{\makecell[c]{\footnotesize{\tiny TB-SAR-D} \\[-5pt]\footnotesize{\tiny $\tau=0.1$}}} &
		\raisebox{.3ex}{\makecell[c]{\footnotesize{\tiny TB-SAR-D} \\[-5pt]\footnotesize{\tiny $\tau=0.3$}}} \\[-5pt]
		
      \end{tabular}
      \caption{Restoration results for non adversarial images and adversarial {samples} given by TB-SAR, TB-SAR-D, and our models under the noise level $L=4$. {Here, a* and b* represent the non adversarial restoration and adversarial restoration, respectively.}}
      \label{te_sar_D_L=4}
\end{figure*}

\section{Conclusion}\label{sec:5}
This study introduces a novel tunable multiplicative denoising network stabilized by {a} diffusion equation.	
Inspired by the dissipative properties of diffusion equations, which facilitate noise-independent denoising, we integrate a linear heat equation within the network to enforce regularization.	
Leveraging the unrolling technique, cumulative regularization is achieved by unfolding our algorithm into a neural network.	
{Our deep neural network offers significantly greater flexibility than traditional networks, allowing the smoothness level to be adjusted by tuning the time step $\tau$ after end-to-end training.}

After proving the stability and convergence of our model,  five other mainstream methods are selected for comparative {experiments on simulated, adversarial, and real-world data}. 
The superior denoising capability of our model, particularly in preserving {the} background, weak edges, and {textures}, is shown in experiments. {The experiments also illustrate the robustness of our network against adversarial attacks.}
Besides, our model exhibits superior visual performance when applied to real SAR images.	
{The results of the ablation study demonstrate the indispensability of the diffusion equation block for adversarial defense. A small time step $\tau$ retains more details and textures, whereas a large time step yields greater regularity and a better denoising effect.}
Finally, model decoupling and {a} comparison with adversarial training illustrate the effectiveness and flexibility of our strategy.

\section*{Acknowledgments}
This work was supported by the National Natural Science Foundation of
China with grant U21B2075, 12171123, 12371419, 12401557, 12101158; the Basic Science Center Project of the National Natural Science Foundation of China with grant 62388101; Natural Sciences Foundation of Heilongjiang Province with grant ZD2022A001; 
the German Research Foundation with grant BU 2327/20-1; 
the China Postdoctoral Science Foundation with grant 2021M701010;
the China Scholarship Council Program with grant 202306120201. 
The authors acknowledge the support from DESY (Hamburg, Germany), a member of the
Helmholtz Association HGF. 
This research was supported in part through the Maxwell computational
resources operated at Deutsches Elektronen-Synchrotron DESY, Hamburg,
Germany. 

\section*{CRediT authorship contribution statement}
\textbf{Yi Ran}: Investigation, Methodology, Validation, Visualization, Writing-original draft, Writing- review \& editing. \textbf{Zhichang Guo}: Conceptualization, Funding acquisition, Methodology, Supervision. \textbf{Jia Li}: Funding acquisition, Writing- review \& editing. \textbf{Yao Li}: Formal analysis, Funding acquisition, Writing- review \& editing. 
\textbf{Martin Burger}: Funding acquisition, Methodology, Supervision, Writing- review \& editing. \textbf{Boying Wu}: Funding acquisition, Supervision.

\section*{Declaration of competing interest}
The authors declare that they have no known competing financial interests or personal relationships that could have appeared to influence the work reported in this paper. 

\section*{Data available}
Data will be made available on request.


\bibliography{UNreference}

\end{document}